\def\year{2021}\relax
\documentclass[letterpaper]{article}
\usepackage{aaai21}
\usepackage{times}
\usepackage{helvet}
\usepackage{courier}
\usepackage[hyphens]{url}
\usepackage{graphicx}
\usepackage{graphicx}
\usepackage{natbib}
\usepackage{caption}
\usepackage[ruled,noline,noalgohanging]{algorithm2e}
\usepackage{stmaryrd}
\usepackage{tikz}
\usetikzlibrary{arrows,positioning,decorations.markings}
\tikzset{
  varnode/.style={rectangle,outer sep=0mm},
  varnodenoperi/.style={rectangle,outer sep=-1mm},
  ourarrow/.style={>=stealth}, 
  ourarc/.style={>=stealth,thick,arc}}
\usepackage{amsthm}
\usepackage{amssymb}

\newtheorem{mylem}{Lemma}
\newtheorem{mydef}{Definition}
\newtheorem{myenc}{Encoding}
\newtheorem{mythm}{Theorem}

\newtheorem{myeg}{Example}
\newtheorem{myconj}{Conjecture}

\usepackage{latexsym}
\usepackage[inline]{enumitem}
\setlist[enumerate]{label=(\roman*)}
\usepackage[Euler]{upgreek}
\usepackage{tabularx}
\usepackage{mathtools}
\usepackage{mathtools}

\title{Computing Plan-Length Bounds Using Lengths of Longest Paths}
\author{}

\author{Mohammad Abdulaziz and Dominik Berger\\
}
\affiliations{Techniche Universit\"at M\"unchen, Germany}

\begin{document}

\maketitle
\begin{abstract}
We devise a method to exactly compute the length of the longest simple path in factored state spaces, like state spaces encountered in classical planning.
Although the complexity of this problem is NEXP-hard, we show that our method can be used to compute practically useful upper-bounds on lengths of plans.
We show that the computed upper-bounds are significantly (in many cases, orders of magnitude) better than bounds produced by previous bounding techniques and that they can be used to improve the SAT-based planning.
\end{abstract}

\providecommand{\insts}{}
\renewcommand{\insts}{\ensuremath{\Delta}}
\providecommand{\inst}{\ensuremath{\tvsal}}
\newcommand{\act}{\ensuremath{\pi}}
\newcommand{\asarrow}[1]{\vec{#1}}
\renewcommand{\vec}[1]{\overset{\rightarrow}{#1}}
\newcommand{\as}{\ensuremath{\vec{{\act}}}}
\newcommand{\asb}{\ensuremath{\vec{{\act_2}}}}

\newcommand{\etc}{\textit{etc.}}
\newcommand{\versus}{\textit{vs.}}
\newcommand{\ie}{i.e.}
\newcommand{\Ie}{I.e.}
\newcommand{\eg}{e.g.}
\newcommand{\michael}[1]{\textcolor{blue}{M: #1}}
\newcommand{\abziz}[1]{\textcolor{brown}{#1}}
\newcommand{\sublist}[2]{ \ensuremath{#1} \preceq\!\!\!\raisebox{.4mm}{\ensuremath{\cdot}}\; \ensuremath{#2}}
\newcommand{\subscriptsublist}[2]{\ensuremath{#1}\preceq\!\raisebox{.05mm}{\ensuremath{\cdot}}\ensuremath{#2}}
\newcommand{\PLS}{\Pi^\preceq\!\raisebox{1mm}{\ensuremath{\cdot}}}
\newcommand{\PLScharles}{\Pi^d}
\newcommand{\execname}{\mathsf{ex}}
\newcommand{\IndHyp}{\mathsf{IH}}
\newcommand{\exec}[2]{#2(#1)}

\newcommand{\ancestorssymbol}{\textsf{\upshape ancestors}}
\newcommand{\ancestors}{\ancestorssymbol}
\newcommand{\satpreas}[2]{\ensuremath{sat_precond_as(s, \as)}}
\newcommand{\proj}[2]{\ensuremath{#1{\downharpoonright}_{#2}}}
\newcommand{\dep}[3]{\ensuremath{#2 {\rightarrow} #3}}
\newcommand{\deptc}[3]{\ensuremath{#2 {\rightarrow^+} #3}}
\newcommand{\negdep}[3]{\ensuremath{#2 \not\rightarrow #3}}
\newcommand{\leavessymbol}{\textsf{\upshape leaves}}
\newcommand{\leaves}{\leavessymbol}

\newcommand{\childrensymbol}{\textsf{\upshape children}}
\newcommand{\children}[2]{\mathcal{\childrensymbol}_{#2}(#1)}
\newcommand{\succsymbol}{\textsf{\upshape succ}}
\newcommand{\succstates}[2]{\succsymbol(#1, #2)}
\newcommand{\concat}{\#}
\newcommand{\RG}{\cite{Rintanen:Gretton:2013}\ }
\newcommand{\KG}{Kovacs' grammar}
\newcommand{\cupdot}{\charfusion[\mathbin]{\cup}{\cdot}}
\newcommand{\cuparrow}{\charfusion[\mathbin]{\cup}{{\raisebox{.5ex} {\smathcalebox{.4}{\ensuremath{\leftarrow}}}}}}
\newcommand{\bigcuparrow}{\charfusion[\mathop]{\bigcup}{\leftarrow}}
\newcommand{\finiteunion}{\cuparrow}
\newcommand{\finitemap}{\ensuremath{\sqsubseteq}}
\newcommand{\dgraph}{dependency graph}
\newcommand{\domain}[1]{{\sc #1}}
\newcommand{\solver}[1]{{\sc #1}}
\providecommand{\problem}[1]{\domain{#1}}
\renewcommand{\v}{\ensuremath{\mathit{v}}}
\providecommand{\vs}[1]{\domain{#1}}
\renewcommand{\vs}{\ensuremath{\mathit{vs}}}
\newcommand{\VS}{\ensuremath{\mathit{VS}}}
\newcommand{\Aut}{\ensuremath{\mathit{Aut}}}
\newcommand{\Inst}[2]{\ensuremath{\mathit{#2 \rightarrow_{#1} #1}}}
\newcommand{\Image}{\ensuremath{\mathit{Im}}}
\newcommand{\Img}[2]{\protect{#1 \llparenthesis #2 \rrparenthesis}}
\newcommand{\SND}{\ensuremath{\mathit{\pi_2}}}
\newcommand{\FST}{\ensuremath{\mathit{\pi_1}}}
\newcommand{\tvsal}{{\pitchfork}}
\newcommand{\nauty}{CGIP}

\newcommand{\pwinter}{\ensuremath{\mathit{\bigcap_{pw}}}}

\newcommand{\dom}{\ensuremath{\mathit{\mathcal{D}}}}
\newcommand{\codom}{\ensuremath{\mathcal{R}}}

\newcommand{\map}{\ensuremath{\mathit{map}}}
\newcommand{\BIJEC}{\ensuremath{\mathit{bij}}}
\newcommand{\INJ}{\ensuremath{\mathit{inj}}}
\newcommand{\funion}{\ensuremath{\overset{\leftarrow}{\cup}}}

\newcommand{\ifnew}{\mbox{\upshape \textsf{if}}}
\newcommand{\thennew}{\mbox{\upshape \textsf{then}}}
\newcommand{\elsenew}{\mbox{\upshape \textsf{else}}}
\newcommand{\choice}{\ensuremath{\epsilon}}
\newcommand{\arbchoice}{\mbox{\upshape \textsf{arb}}}
\newcommand{\acycchoice}{\mbox{\upshape \textsf{ac}}}
\newcommand{\cycchoice}{\mbox{\upshape \textsf{cyc}}}
\newcommand{\filter}{\ensuremath{\mathit{FIL}}}
\newcommand{\probset}{\ensuremath{\boldsymbol \Pi}}
\newcommand{\probleq}{\ensuremath{\leq_\Pi}}
\newcommand{\CommVar}{\ensuremath{\bigcap_\v} }
\newcommand{\quotfun}{\ensuremath{ \mathcal{Q}}}

\newcommand{\apre}{\mbox{\upshape \textsf{pre}}}
\newcommand{\aeff}{\mbox{\upshape \textsf{eff}}}
\newcommand{\problist}{\ensuremath \probset}
\newcommand{\cat}{{\frown}}
\newcommand{\probproj}[2]{{#1}{\downharpoonright}^{#2}}
\newcommand{\preced}{\mathbin{\rotatebox[origin=c]{180}{\ensuremath{\rhd}}}}
\newcommand{\perm}{\ensuremath{\sigma}}
\newcommand{\invariant}[2]{\ensuremath{\mathit{inv({#1},{#2})}}}
\newcommand{\invstates}[1]{\ensuremath{\mathit{inv({#1})}}}
\newcommand{\probss}[1]{{\mathcal S}(#1)}
\newcommand{\parChildRel}[3]{\ensuremath{\negdep{#1}{#2}{#3}}}
\newcommand{\asessymbol}{\ensuremath{\mathbb{A}}}
\newcommand{\ases}[1]{{#1}^*}
\newcommand{\uniStates}{\ensuremath{\mathbb{U}}}
\newcommand{\recurrenceDiam}{\ensuremath{\mathit{rd}}}
\newcommand{\recurrenceAcycDiamfun}{\ensuremath{\mathit{{\mathfrak A}}}}
\newcommand{\recurrenceDiamfun}{\ensuremath{\mathit{\mathfrak R}}}
\newcommand{\traversalDiam}{\ensuremath{\mathit{td}}}
\newcommand{\traversalDiamfun}{\ensuremath{\mathit{\mathfrak T}}}
\newcommand{\isPrefix}[2]{\ensuremath{#1 \preceq #2}}
\providecommand{\path}{\ensuremath{\gamma}}
\newcommand{\aspath}{\ensuremath{\vec{\path}}}
\renewcommand{\path}{\ensuremath{\gamma}}
\newcommand{\n}{\textsf{\upshape n}}
\providecommand{\graph}{}
\providecommand{\cal}{}
\renewcommand{\cal}{}
\renewcommand{\graph}{{\cal G}}
\newcommand{\undirgraph}{{\cal G}}
\newcommand{\sset}{\ensuremath{\mbox{\upshape \textsf{ss}}}}
\renewcommand{\ss}{\ensuremath{\state s}}
\newcommand{\slist}{\ensuremath{\vec{\mbox{\upshape \textsf{ss}}}}}
\newcommand{\sll}{\ensuremath{\vec{\state}}}
\newcommand{\listset}{\mbox{\upshape \textsf{set}}}
\newcommand{\asset}{\ensuremath{\mathit{K}}}
\newcommand{\aslist}{\ensuremath{\mathit{\overset{\rightarrow}{\gamma}}}}
\newcommand{\head}{\mbox{\upshape \textsf{hd}}}
\renewcommand{\max}{\textsf{\upshape max}}
\newcommand{\argmax}{\textsf{\upshape argmax}}
\renewcommand{\min}{\textsf{\upshape min}}
\newcommand{\bool}{\mbox{\upshape \textsf{bool}}}
\newcommand{\last}{\mbox{\upshape \textsf{last}}}
\newcommand{\front}{\mbox{\upshape \textsf{front}}}
\newcommand{\rot}{\mbox{\upshape \textsf{rot}}}
\newcommand{\stuff}{\mbox{\upshape \textsf{intlv}}}
\newcommand{\tail}{\mbox{\upshape \textsf{tail}}}
\newcommand{\ngrtoas}{\ensuremath{\mathit{\as_{\graph_\mathbb{N}}}}}
\newcommand{\vsfun}{\mbox{\upshape \textsf{vs}}}
\newcommand{\inits}{\mbox{\upshape \textsf{init}}}
\newcommand{\satprecondas}{\mbox{\upshape \textsf{sat-pre}}}
\newcommand{\remcondlessact}{\mbox{\upshape \textsf{rem-cless}}}
\providecommand{\state}{}
\renewcommand{\state}{x}
\newcommand{\statea}{\ensuremath{x_1}}
\newcommand{\stateb}{\ensuremath{x_2}}
\newcommand{\statec}{\ensuremath{x_3}}
\newcommand{\fals}{\mbox{\upshape \textsf{F}}}
\newcommand{\indices}{\ensuremath{V}}
\newcommand{\edges}{\ensuremath{E}}
\newcommand{\vertices}{\ensuremath{V}}
\newcommand{\listtype}{\mbox{\upshape \textsf{list}}}
\newcommand{\settype}{\mbox{\upshape \textsf{set}}}
\newcommand{\acttype}{\mbox{\upshape \textsf{action}}}
\newcommand{\graphtype}{\mbox{\upshape \textsf{graph}}}
\newcommand{\projfun}[2]{\ensuremath{\Delta_{#1}^{#2}}}
\newcommand{\snapfun}[2]{\ensuremath{\Sigma_{#1}^{#2}}}
\newcommand{\RDfun}[1]{\ensuremath{{\mathcal R}_{#1}}}
\newcommand{\elldbound}[1]{\ensuremath{{\mathcal LS}_{#1}}}
\newcommand{\distinct}{\textsf{\upshape distinct}}
\newcommand{\ddistinct}{\mbox{\upshape \textsf{ddistinct}}}
\newcommand{\simple}{\mbox{\upshape \textsf{simple}}}

\newcommand{\reachable}[3]{\ensuremath{{#1}\rightsquigarrow{#3}}}

\newcommand{\Omit}[1]{}

\newcommand{\charles}[1]{\textcolor{red}{#1}}
\newcommand{\mohammad}[1]{Mohammad: \textcolor{green}{#1}}

\newcommand{\negreachable}[3]{\ensuremath{{#2}\not\rightsquigarrow{#3}}}
\newcommand{\wdiam}[2]{{#1}^{#2}}
\newcommand{\dsnapshot}[2]{\Delta_{#1}}
\newcommand{\ellsnapshot}[2]{{\mathcal L}_{#1}}

\newcommand{\snapshotsymbol}{|\kern-.7ex\raise.08ex\hbox{\scalebox{0.7}{$\bullet$}}}
\newcommand{\snapshot}[2]{\ensuremath{\mathrel{#1\snapshotsymbol_{#2}}}}
\newcommand{\vstype}{\texttt{\upshape VS}}
\newcommand{\vtype}{{\scriptsize \ensuremath{\dom(\delta)}}}
\newcommand{\Balgo}{{\mbox{\textsc{Hyb}}}}
\newcommand{\ssgraph}[1]{\graph_\ss}
\newcommand{\agree}{\textsf{\upshape agree}}
\newcommand{\ck}{\ensuremath{\texttt{ck}}}
\newcommand{\lk}{\ensuremath{\texttt{lk}}}
\newcommand{\gr}{\ensuremath{\texttt{gr}}}
\newcommand{\gk}{\ensuremath{\texttt{gk}}}
\newcommand{\CK}{\ensuremath{\texttt{CK}}}
\newcommand{\LK}{\ensuremath{\texttt{LK}}}
\newcommand{\GR}{\ensuremath{\texttt{GR}}}
\newcommand{\GK}{\ensuremath{\texttt{GK}}}
\newcommand{\safe}{\ensuremath{\texttt{s}}}

\newcommand{\derivname}{\ensuremath{\partial}}
\newcommand{\deriv}[3]{\ensuremath{\derivname(#1,#2,#3)}}
\newcommand{\derivabbrev}[3]{\ensuremath{{\partial(#1,#2)}}}
\newcommand{\subsetoracle}{\ensuremath{ \Omega}}
\newcommand{\Aalgo}{{\mbox{\textsc{Pur}}}}
\newcommand{\Sname}{\textsf{\upshape S}}
\newcommand{\Sbrace}[1]{\Sname\langle#1\rangle}
\newcommand{\SalgoName}{\Sname_{\textsf{\upshape max}}}
\newcommand{\Salgo}[1]{\SalgoName\langle#1\rangle}

\newcommand{\WLPname}{{\mbox{\textsc{wlp}}}}
\newcommand{\WLPbrace}[1]{\WLPname\langle#1\rangle}
\newcommand{\WLPalgoName}{\WLPname_{\textsf{\upshape max}}}
\newcommand{\WLP}[1]{\WLPalgoName\langle#1\rangle}

\newcommand{\Nname}{\ensuremath{\textsf{\upshape N}}}
\newcommand{\Nbrace}[1]{\Nname\langle#1\rangle}
\newcommand{\NalgoName}{\Nname{_{\textsf{\upshape sum}}}}
\newcommand{\Nalgobrace}[1]{\NalgoName\langle#1\rangle}

\newcommand{\acycNname}{\widehat{\textsf{\upshape N}}}
\newcommand{\acycNbrace}[1]{\acycNname\langle#1\rangle}
\newcommand{\acycNalgoName}{\acycNname{_{\textsf{\upshape sum}}}}
\newcommand{\acycNalgobrace}[1]{\acycNalgoName\langle#1\rangle}

\newcommand{\Mname}{\ensuremath{\textsf{\upshape M}}}
\newcommand{\Mbrace}[1]{\Mname\langle#1\rangle}
\newcommand{\MalgoName}{\Mname{_{\textsf{\upshape sum}}}}
\newcommand{\Malgobrace}[1]{\MalgoName\langle#1\rangle}
\newcommand{\cardinality}[1]{{\ensuremath{|#1|}}}
\newcommand{\length}[1]{\cardinality{#1}}
\newcommand{\basecasefun}{\ensuremath{b}}
\newcommand{\Basecasefun}{\ensuremath{\mathcal B}}

\newcommand{\vertexgen}{\ensuremath{u}}
\newcommand{\vertexa}{{\ensuremath{\vertexgen_1}}}
\newcommand{\vertexb}{{\ensuremath{\vertexgen_2}}}
\newcommand{\vertexc}{{\ensuremath{\vertexgen_3}}}
\newcommand{\vertexd}{{\ensuremath{\vertexgen_4}}}
\newcommand{\vertexsetgen}{\ensuremath{\mathit{us}}}
\newcommand{\vertexseta}{\vertexsetgen_1}
\newcommand{\vertexsetb}{\vertexsetgen_2}
\newcommand{\labelsymbol}{\ensuremath{l}}
\newcommand{\labelfun}{\ensuremath{\mathcal{L}}}
\newcommand{\DAG}{\ensuremath{A}}
\newcommand{\NalgoNameN}{{\ensuremath{\NalgoName_{\mathbb{N}}}}}
\newcommand{\NnameN}{\ensuremath{\Nname_\mathbb{N}}}
\newcommand{\replaceprojsinglename}{\raisebox{-0.3mm} {\scalebox{0.7}{\textpmhg{H}}}}
\newcommand{\replaceprojsingle}[3] {{ #2} \underset {#1} {\raisebox{-0.3mm} {\scalebox{0.7}{\textpmhg{H}}}}  #3}
\newcommand{\HOLreplaceprojsingle}[1]{\underset {#1} {\raisebox{-0.3mm} {\scalebox{0.7}{\textpmhg{H}}}}}

\newcommand{\lotus}{{\scalebox{0.6}{\includegraphics{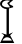}}}}
\newcommand{\invlotus}{\mathbin{\rotatebox[origin=c]{180}{$\lotus$}}}
\newcommand{\clique}{\ensuremath{K}}
\newcommand{\partition}{\ensuremath{\vs_{1..n}}}
\newcommand{\partitiontype}{\ensuremath{\vstype_{1..n}}}
\newcommand{\vtxpartition}{\ensuremath{P}}

\newcommand{\traversalDiamAlgo}{{\mbox{\textsc{TravDiam}}}}
\newcommand{\prefix}{\textsf{\upshape pfx}}
\newcommand{\powerset}{\mathbb{P}}
\newcommand{\postfix}{\textsf{\upshape sfx}}
\newcommand{\dfunproj}{\ensuremath{{\mathfrak D}}}
\newcommand{\dfunsnap}{\ensuremath{{\textgoth D}}}
\newcommand{\ellfunproj}{\ensuremath{\mathfrak L}}
\newcommand{\ellfunsnap}{\ensuremath{\textgoth L}}
\newcommand{\cycle}{\ensuremath{C}}
\newcommand{\petal}{\ensuremath{\eta}}
\renewcommand{\prod}{\ensuremath{{{{{\mathlarger{\mathlarger {{\mathlarger {\Pi}}}}}}}}}}
\newcommand{\sccset}{{\ensuremath{SCC}}}
\newcommand{\scc}{{\ensuremath{scc}}}
\newcommand{\negate}[1]{\overline{#1}}
\newcommand{\setofsets}{\ensuremath{S}}
\newcommand{\group}{\ensuremath{\cal \Gamma}}
\newcommand{\neededvars}{{\cal N}}
\newcommand{\sspace}{\mbox{\upshape \textsf{sspc}}}
\newcommand{\tip}{\ensuremath{t}}
\newcommand{\vara}{\ensuremath{\v_1}}
\newcommand{\varb}{\ensuremath{\v_2}}
\newcommand{\varc}{\ensuremath{\v_3}}
\newcommand{\vard}{\ensuremath{\v_4}}
\newcommand{\vare}{\ensuremath{\v_5}}
\newcommand{\varf}{\ensuremath{\v_6}}
\newcommand{\varg}{\ensuremath{\v_7}}
\newcommand{\varh}{\ensuremath{\v_8}}
\newcommand{\vari}{\ensuremath{\v_9}}
\newcommand{\acta}{\ensuremath{\act_1}}
\newcommand{\actb}{\ensuremath{\act_2}}
\newcommand{\actc}{\ensuremath{\act_3}}
\newcommand{\actd}{\ensuremath{\act_4}}
\newcommand{\acte}{\ensuremath{\act_5}}
\newcommand{\actf}{\ensuremath{\act_6}}
\newcommand{\actg}{\ensuremath{\act_7}}
\newcommand{\acth}{\ensuremath{\act_8}}
\newcommand{\acti}{\ensuremath{\act_9}}

\newcommand{\planningproblem}{\Uppi}

\tikzset{dots/.style args={#1per #2}{line cap=round,dash pattern=on 0 off #2/#1}}
\providecommand{\moham}[1]{\fbox{{\bf \@Mohammad: }#1}}
\newcommand{\TDbound}{{\mbox{\textsc{Arb}}}}
\newcommand{\expbound}{{\mbox{\textsc{Exp}}}}
\newcommand{\sasdom}{\expbound}
\newcommand{\cardfun}{\ensuremath{\mathbb{C}}}
\newcommand{\AGNa}{AGN1}
\newcommand{\AGNb}{AGN2}
\newcommand{\reset}{{\ensuremath{reset}}}
\newcommand{\cost}{{\ensuremath{\mathcal{C}}}}
\newcommand{\goal}{{\ensuremath{\mathcal{G}}}}
\newcommand{\init}{{\ensuremath{\mathcal{I}}}}
\newcommand{\completenessthreshold}{{\ensuremath{\mathcal{CT}}}}
\newcommand{\subsetDiam}{\mathscr{S}}
 \renewcommand{\childrensymbol}{\textsf{\upshape child}}
\renewcommand{\traversalDiamAlgo}{{\mbox{\textsc{TravD}}}}

\newcommand{\biere}{BiereCCZ99}
\section{Introduction}
\label{sec:intro}

Many techniques for solving problems defined on transition systems, like SAT-based planning~\cite{kautz:selman:92} and bounded model checking~\cite{BiereCCZ99}, benefit from knowledge of {\em upper bounds} on the lengths of solution transition sequences, aka \emph{completeness thresholds}.
If $N$ is such a bound, and if a solution exists, then that solution need not comprise more than $N$ transitions.

In AI planning, upper bounds on plan lengths have two main uses related to SAT-based planning.
Firstly, like for bounded model-checking, an upper bound on plan lengths can be used as a completeness threshold, i.e. to prove a planning problem has no solution.
Secondly, it can be used to improve the ability of a SAT-based planner to find a solution.
Typically, a SAT-based planner queries a SAT solver to search for plans of increasing lengths, aka horizons, going through many unsatisfiable formulae until the given horizon is longer than the shortest possible plan.
If a horizon longer than the shortest plan were initially provided, the planner can avoid many of the costly unsatisfiable queries~\cite{DBLP:journals/aicom/GereviniSV15}.
Using plan length upper bounds as horizons in this way 
was shown to increase the coverage of SAT-based planners~\cite{Rintanen:Gretton:2013,icaps2017,abdulaziz:2019}.

\citeauthor{BiereCCZ99} identified the {\em diameter} ($d$) and the \emph{recurrence diameter} ($\recurrenceDiam$), which are topological properties of the state space, as completeness thresholds for bounded model-checking of safety and liveness properties, respectively.
$d$ is the longest shortest path between any two states.
$\recurrenceDiam$ is the length of the longest simple path in the state space, i.e. the length of the longest path that does not traverse any state more than once.
Both, $d$ and $\recurrenceDiam$, are upper bounds on the shortest plan's length, i.e. they are completeness thresholds for SAT-based planning.
Also, $d$ is a lower bound on $\recurrenceDiam$ that can be exponentially smaller.

Computing $d$ or $\recurrenceDiam$ for succinctly represented transition systems, such as \emph{factored systems}, is hindered by the worst-case complexity of all existing methods, which is exponential or doubly-exponential in the size of the given system, respectively.
This complexity can be alleviated by \emph{compositionally} computing upper bounds on $d$ or $\recurrenceDiam$ instead of exactly computing them.
Compositional bounding methods compute an upper bound on a factored transition system's diameter by composing together values of topological properties of state spaces of abstract subsystems~\cite{baumgartner2002property,Rintanen:Gretton:2013,abdulaziz2015verified,icaps2017,abdulaziz:2019}, and they are currently the only practically viable method to compute bounds on plan lengths or the state space diameter.
Compositional approaches provide useful bounds on plan length or the diameter using potentially exponentially smaller computational effort compared to directly computing $d$ or $\recurrenceDiam$, since explicit representations of only abstract subsystems have to be constructed.

In this work we study the computation of $\recurrenceDiam$ for state spaces of classical planning problems, which are factored transition systems.
The longest simple path and its length are fundamental graph properties.
Thus, computing $\recurrenceDiam$ for state spaces of planning problems is inherently interesting, as it might reveal interesting properties of different planning problems.
However, our goal is to devise better compositional methods to compute upper bounds on plan lengths to aid SAT-based planning.
This raises an interesting question: why should we focus on computing $\recurrenceDiam$ instead of $d$?
This question is reasonable since, as stated earlier, $d$ can be computed in exponentially less time than $\recurrenceDiam$, and $\recurrenceDiam$ is an upper bound on $d$ that can be exponentially larger.
The reason is simple: it has been shown that $d$ cannot be bounded by diameters of \emph{projections}~\cite[Chapter 3, Theorem 1]{abdulaziz2017formally}.
Since projections are cornerstone abstractions for compositional bounding, a method to compute $d$ cannot be leveraged for compositional bounding.
On the other hand, previous authors showed that, theoretically, recurrence diameters of projections can be composed to bound the concrete system's diameter~\cite{baumgartner2002property,icaps2017}.
Here we explore the potential of this theoretical possibility: we study methods to compute $\recurrenceDiam$ and the use of those methods to improve existing compositional bounding algorithms.

Our first contribution concerns the relationship between $\recurrenceDiam$ and the \emph{traversal diameter} ($\traversalDiam$), which is another state space topological property.
The best existing compositional bounding method is due to \citeauthor{abdulaziz:2019}~\citeyear{abdulaziz:2019}, and it computes traversal diameters of abstractions and composes them into an upper bound on $d$ and on plan-length.
We show that $\traversalDiam$ is an upper bound on $\recurrenceDiam$, and that $\recurrenceDiam$ can be exponentially smaller than $\traversalDiam$.
This gives an opportunity for substantial improvements in the bounds computed by compositonal bounding methods, if recurrence diameters of abstractions are used instead of their traversal diameters.
However, the practical realisation of this improvement is contingent on whether there is an efficient method to compute $\recurrenceDiam$.

Our second, and main, contribution is that we investigate practically useful methods to compute recurrence diameters of factored systems that come from planning problems.
We implement two methods to compute recurrence diameters based on the work of \citeauthor{BiereCCZ99}~\citeyear{BiereCCZ99}.
Unlike \citeauthor{BiereCCZ99}, who devised their method and tested it on model-checking problems, we test these methods on planning benchmarks and show that it is impractical for most planning problems.
We show, however, that computing the recurrence diameter within the compositional bounding method by \citeauthor{icaps2017}~\citeyear{icaps2017} leads to bounds that are, as predicted theoretically, much tighter when $\recurrenceDiam$ is used instead of $\traversalDiam$.

However, a challenge is that our method to compute $\recurrenceDiam$ and that of \citeauthor{BiereCCZ99} have wort-case running times that are doubly-exponential in the size of the given factored system.
This is much worse than the worst-case complexity of computing $\traversalDiam$, which is singly-exponential in the size of the factored system.
This stems from the complexity of the problem of computing $\recurrenceDiam$ for succinct digraphs, which is NEXP-hard~\cite{pardalos2004note,papadimitriou1986note}.
Our third contribution is that we investigate techniques to alleviate the impact of this prohibitive worst-case running time on the overall compositional bounding algorithm by combining the computation of $\recurrenceDiam$ and $\traversalDiam$.

Lastly, we experimentally show that the improved bounds lead to an improved problem coverage for state-of-the-art SAT-based planner {\sc Mp}~\cite{rintanen:12}, when the bounds are used as horizons for it.

\section{Background and Notation}
\label{sec:defs}

We consider {\em factored transition systems} which are characterised in terms of a set of {\em actions}.
From actions we can define a set of {\em valid states}, and then approach bounds by considering properties of {\em executions} of actions on valid states.
Whereas conventional expositions in the planning and model-checking literature would also define initial conditions and goal/safety criteria, here we omit those features from discussion since the 
state-space topological properties we consider are independent of those features.

\begin{mydef}[States and Actions]
\label{def:stateAction}
A maplet, $\v \mapsto b$, maps a variable $\v$---i.e. a state-characterising proposition---to a Boolean $b$.
A state, $\state$, is a finite set of maplets.
We write $\dom(\state)$ to denote $\{\v \mid (\v \mapsto b) \in \state\}$, the domain of $\state$.
For states $\state_1$ and $\state_2$, the union, $\state_1 \uplus \state_2$, is defined as $\{\v \mapsto b \mid $ $\v \in \dom(\state_1) \cup \dom(\state_2) \wedge \ifnew\; \v \in \dom(\state_1) \;\thennew\; b = \state_1(\v) \;\elsenew\; b = \state_2(\v)\}$.
Note that the state $\state_1$ takes precedence.
An action is a pair of states, $(p,e)$, where $p$ represents the \emph{preconditions} and $e$ represents the \emph{effects}.
For action $\act=(p,e)$, the domain of that action is $\dom(\act)\equiv\dom(p) \cup \dom(e)$.
\end{mydef}
 
\begin{mydef}[Execution]
\label{def:exec}
When an action \act\ $(=(p,e))$ is executed at state $\state$, it produces a successor state $\exec{\state}{\act}$, formally defined as
$\exec{\state}{\act} = \ifnew\; p \nsubseteq \state\; \thennew\; \state \;\elsenew\; e \uplus \state$.
We lift execution to lists of actions $\as$, so $\exec{\state}{\as}$ denotes the state resulting from successively applying each action from $\as$ in turn, starting at $\state$.
\end{mydef} We give examples of states and actions using sets of literals, where we denote the maplet $a\mapsto \top$ with the literal $a$ and $a\mapsto \bot$ with the literal $\overline{a}$.
For example, $(\{a,\overline{b}\}, \{c\})$ is an action that if executed in a state where $a$ is true and $b$ is false, it sets $c$ to true.
$\dom((\{a,\overline{b}\}, \{c\})) = \{a,b,c\}$.
We also give examples of sequences, which we denote by the square brackets, e.g. $[a,b,c]$.

\begin{mydef}[Factored Transition System]
\label{def:actoredTransitionSystem}
A set of actions $\delta$ constitutes a factored transition system.
$\dom(\delta)$ denotes the domain of $\delta$, which is the union of the domains of all the actions in $\delta$.
Let $\listset(\as)$ be the set of elements in $\as$.
The set of valid action sequences, $\ases{\delta}$, is $\{\as \mid\; \listset(\as) \subseteq \delta\}$.
The set of valid states, $\uniStates(\delta)$,  is $\{\state \mid \dom(\state) = \dom (\delta)\}$.
$\graph(\delta)$ denotes the set of pairs 
$\{(\state, \exec{\state}{\act})\mid \state \in \uniStates(\delta), \act \in \delta\}$, which is all non self-looping transitions in the state space of $\delta$.
\end{mydef}
 {\makeatletter
\def\old@comma{,}
\catcode`\,=13
\def,{\ifmmode\old@comma\discretionary{}{}{}\else\old@comma\fi}
\makeatother
\makeatletter
\def\old@dot{.}
\catcode`\.=13
\def.{\ifmmode\old@dot\discretionary{}{}{}\else\old@dot\fi}
\makeatother
 \begin{figure}[t]
\begin{minipage}[b]{0.15\textwidth}
\centering
\begin{tikzpicture}[scale=0.8]
\node (s0) at (4,1) [varnodenoperi] {\footnotesize $\overline{\vara\varb}$ } ;
\node (s2) at (3,0)   [varnodenoperi] {\footnotesize $\overline{\vara}\varb$ } ;
\node (s4) at (5,0)   [varnodenoperi] {\footnotesize $\vara\overline{\varb}$ } ;
\node (s6) at (4,-1)[varnodenoperi] {\footnotesize $\vara\varb$ } ;
\draw [->] (s0) -- (s2) ;
\draw [->] (s2) -- (s0) ;
\draw [->] (s0) -- (s4) ;
\draw [->] (s4) -- (s0) ;
\draw [->] (s0) -- (s6) ;
\draw [->] (s6) -- (s0) ;
\draw [->] (s6) -- (s2) ;
\draw [->] (s2) -- (s6) ;
\draw [->] (s4) -- (s2) ;
\draw [->] (s2) -- (s4) ;
\draw [->] (s4) -- (s6) ;
\draw [->] (s6) -- (s4) ;
\end{tikzpicture}
\caption[fig:stateDiag]{\label{fig:stateDiag}}
\end{minipage}
\begin{minipage}[b]{0.15\textwidth}
\centering
\begin{tikzpicture}[scale=0.8]
\node (s0) at (4,1) [varnodenoperi] {\footnotesize $\overline{\vara\varb}$ } ;
\node (s1) at (3,0)   [varnodenoperi] {\footnotesize $\overline{\vara}\varb$ } ;
\node (s2) at (5,0)   [varnodenoperi] {\footnotesize $\vara\overline{\varb}$ } ;
\node (s3) at (4,-1)[varnodenoperi] {\footnotesize $\vara\varb$ } ;
\draw [->] (s0) -- (s1) ;
\draw [->] (s0) -- (s2) ;
\draw [->] (s0) -- (s3) ;
\end{tikzpicture}
\caption[fig:stateDiag2]{\label{fig:stateDiag2}}
\end{minipage}
\begin{minipage}[b]{0.16\textwidth}
\centering
\begin{tikzpicture}[scale=0.8]
\node (s0) at (4,1) [varnodenoperi] {\footnotesize $\overline{\vara\varb}$ } ;
\node (s1) at (3,0)   [varnodenoperi] {\footnotesize $\overline{\vara}\varb$ } ;
\node (s2) at (5,0)   [varnodenoperi] {\footnotesize $\vara\overline{\varb}$ } ;
\node (s3) at (4,-1)[varnodenoperi] {\footnotesize $\vara\varb$ } ;
\draw [<->] (s0) -- (s1) ;
\draw [<->] (s0) -- (s2) ;
\draw [<->] (s0) -- (s3) ;
\end{tikzpicture}
\caption[fig:stateDiag2]{\label{fig:stateDiag3}}
\end{minipage}
\caption{The state spaces of the systems from Examples~\ref{eg:factored},~\ref{eg:td},~and~\ref{eg:tdexpltrd}.
}
\end{figure}

\begin{myeg}
\label{eg:factored}
Consider the factored system $\delta = \{\acta =(\emptyset, \{\vara, \varb\}),
        \actb =(\emptyset, \{\overline{\vara}, \varb\}),
        \actc =(\emptyset, \{\vara, \overline{\varb}\}),
        \actd =(\emptyset, \{\overline{\vara}, \overline{\varb}\})\}$.
Figure~\ref{fig:stateDiag} shows $\graph(\delta)$, i.e.\ the state space of $\delta$, where different states defined on the variables $\dom(\delta) = \{\vara,\varb\}$ are shown.
Since every state can be reached via one action from every other state, the state space is a clique.
\end{myeg}
 }

For a system $\delta$, a bound on the length of action sequences is $\expbound(\delta) = 2^\cardinality{\dom(\delta)} - 1$ (i.e.\ one less than the number of valid states), where $\cardinality{\bullet}$ denotes the cardinality of a set or the length of a list.
Other bounds employed by previous approaches are topological properties of the state space.
One such topological property is the diameter, suggested by \citeauthor{BiereCCZ99}~\citeyear{BiereCCZ99}, which is the length of the longest shortest path between any two states in the state space of a system.
\begin{mydef}[Diameter]
\label{def:diam}
The diameter, written $d(\delta)$, is the length of the longest shortest action sequence, formally
\[d(\delta) = \underset{\state \in \uniStates(\delta), \as \in \ases{\delta}}{\max} \;\;\; \underset{\exec{\state}{\as} = \exec{\state}{\as'}, \as' \in \ases{\delta} }{\min}\; \cardinality{\as'}\]
\end{mydef}
 Note that if there is a valid action sequence between any two valid states of $\delta$, then there is a valid action sequence between them which is not longer than $d(\delta)$.
Thus it is a completeness threshold for bounded model-checking and SAT-based planning.
Another topological property that is an upper bound on plan lengths is the recurrence diameter, which is the length of the longest simple path in the state space of a transition system.
It was proposed by \citeauthor{BiereCCZ99}~\citeyear{BiereCCZ99}.
\begin{mydef}[Recurrence Diameter]
\label{def:diamrd}
Let $\distinct(\state,\as)$ denote that all states traversed by executing $\as$ at $\state$ are distinct states.
The recurrence diameter is the length of the longest simple path in the state space, formally
\[
\recurrenceDiam(\delta) = \underset{\state \in \uniStates(\delta),\as \in \ases{\delta}, \distinct(\state,\as)}{\max}\;\; \cardinality{\as}
\]
\end{mydef}
 \begin{myeg}
\label{eg:dexpltrd}
For the system $\delta$ from Example~\ref{eg:factored}, $d(\delta) = 1$, since every state can be reached with one action from every other state.
Nonetheless, $\recurrenceDiam(\delta) = 3$ as there are many paths with 3 actions in the state space that traverse distinct states, e.g. executing the action sequence $[\acta,\actb,\actc]$ at the state $\{\overline{\v_1},\overline{\v_2}\}$ traverses the distinct states $[\{\overline{\v_1},\overline{\v_2}\}, \{\v_1,\v_2\}, \{\overline{\v_1},\v_2\}, \{\v_1,\overline{\v_2}\}]$.
\end{myeg}
Note that in general $\recurrenceDiam$ is an upper bound on $d$, and that it can be exponentially larger than $d$.

\begin{mythm}[\citeauthor{BiereCCZ99}~\citeyear{BiereCCZ99}]
\label{thm:rdgediam}
For any system $\delta$, we have that $d(\delta) \leq \recurrenceDiam(\delta)$.
Also, there are infinitely many systems for which the recurrence diameter is exponentially (in the number of state variables) larger than the diameter.
\end{mythm}
Like $d$, $\recurrenceDiam$ is a completeness threshold for SAT-based planning and for safety bounded model-checking but, unlike $d$, $\recurrenceDiam$ is also a completeness threshold for bounded model-checking of liveness properties, which was the original reason for its inception~\cite{BiereCCZ99}.

Algorithms have been developed to calculate both properties for digraphs, and those algorithms can be directly applied to state spaces of explicitly represented (e.g. tabular) transition systems.
Exact algorithms to compute $d$ have worse than quadratic runtimes in the number of states~\cite{fredman1976new,alon1997exponent,chan2010more,yuster2010computing}, and approximation algorithms have super-linear runtimes~\cite{aingworth1999fast,roditty2013fast,chechik2014better,abboud2016approximation}. 
The situation is worse for $\recurrenceDiam$, whose computation is NP-hard~\cite{pardalos2004note} for explicitly represented systems.
The impracticality of computing $d$ and $\recurrenceDiam$ is exacerbated in settings where transition systems are described using factored representations, like in planning and model-checking~\cite{fikes1971strips,mcmillan1993symbolic}.
In particular, the worst-case running times are exponentially worse because, in the worst case, all known methods construct an explicit representation of the state space to compute $d$ or $\recurrenceDiam$ of the state space of a succinctly represented system.
This follows the general pattern of complexity exponentiation of graph problems when graphs are succinctly represented, where, for succinct digraphs, the complexity of computing $d$ is $\Pi_2^P$-hard~\cite{hemaspaandra2010complexity} and the complexity of computing $\recurrenceDiam$ is NEXP-hard~\cite{papadimitriou1986note}, instead of being in P and NP-hard, respectively, in the explicit case.

\subsection{Compositional Bounding of the Diameter}

The prohibitive complexity of computing $d$ or $\recurrenceDiam$ suggests they can only be feasibly computed for very small factored systems, systems that are much smaller than those that arise in typical classical planning benchmarks.
However, another possibility is to utilise the computation of $d$ or $\recurrenceDiam$ within compositional plan length upper bounding techniques.
Existing techniques compute an upper bound on $d$, for a given system, by computing topological properties of abstractions of the given system and then composing the abstractions' topological properties. 
Those abstractions are usually much smaller than the given concrete system, where their state spaces can be exponentially smaller than the given system's state space.
Thus, computing topological properties of abstractions might be feasible.

Currently, the compositional bounding method by~\citeauthor{icaps2017}~\citeyear{icaps2017} is the most successful in decomposing a given system into the smallest abstractions.
It decomposes a given factored system using two kinds of abstraction: \emph{projection} and \emph{snapshotting}.
Projection~\cite{knoblock:94,williams:nayak:97} produces an over-approximation of the given system and it was used for bounding by many previous authors.
Snapshotting produces an under-approximation of the given system and it was introduced by~\citeauthor{icaps2017}~\citeyear{icaps2017}.
The compositional method devised by~\citeauthor{icaps2017}~\citeyear{icaps2017} recursively interleaves the application of projection and snapshotting until the system is decomposed into subsystems that can no longer be decomposed, to which we refer here as \emph{base case systems}.
\newcommand{\Balgobrace}[1]{\ensuremath{\Balgo\langle{#1}\rangle}}
After the system is decomposed into base case systems, a topological property, which we call the \emph{base case function}, of the state space of each of the base case systems is computed.
Then the values of the base case function applied to the different base case systems are composed to bound the diameter of the concrete system.

Most authors used the base case function $\expbound$, which is one less than the number of valid states for the given base case system~\cite{Rintanen:Gretton:2013,abdulaziz2015verified,icaps2017}.
A notable exception is \citeauthor{abdulaziz:2019}~\citeyear{abdulaziz:2019}, who used the \emph{traversal diameter}, which is a topological property of the state space, as a base case function.
The traversal diameter is one less than the largest number of states that could be traversed by any path.
\begin{mydef}[Traversal Diameter]
\label{def:td}
Let $\sset(\state,\as)$ be the set of states traversed by executing $\as$ at $\state$.
The traversal diameter is
\[\traversalDiam(\delta) = \underset{\state \in \uniStates(\delta), \as \in \ases{\delta}}{\max}\;\; \cardinality{\sset(\state,\as)} - 1.\]
\end{mydef}
 \begin{myeg}
\label{eg:td}
Consider the factored system $\delta = \{\acta =(\{\overline{\vara}, \overline{\varb}\}, \{\vara, \varb\}),
        \actb =(\{\overline{\vara}, \overline{\varb}\}, \{\overline{\vara}, \varb\}),
        \actc =(\{\overline{\vara}, \overline{\varb}\}, \{\vara, \overline{\varb}\})\}$.
The digraph in Figure~\ref{fig:stateDiag2} shows the state space of $\delta$.
For $\delta$, $\expbound(\delta) = 3$, while $\traversalDiam(\delta) = 1$.
\end{myeg}
 \citeauthor{abdulaziz:2019}~\citeyear{abdulaziz:2019} showed that $\traversalDiam$ is an upper bound on $\recurrenceDiam$ and a lower bound on $\expbound$.
He also showed that $\traversalDiam$ can be exponentially smaller than $\expbound$, as shown in the above example.
This is why, when \citeauthor{abdulaziz:2019}~\citeyear{abdulaziz:2019} used $\traversalDiam$ as a base case function, his method computed substantially tighter bounds than previous methods, which all used $\expbound$.
 \section{The Recurrence Diameter Versus the Traversal Diameter}
We now study the relationship between the recurrence diameter and the traversal diameter.
The core insight we make here is that $\recurrenceDiam$ can be exponentially smaller than $\traversalDiam$.
\begin{mythm}
\label{lem:rdexplttd}
There are infinitely many factored systems whose recurrence diameters are exponentially smaller (in the number of state variables) than their traversal diameters.
\end{mythm}
 \renewcommand{\probss}{\mathcal S}
\begin{proof}
Let, for a natural number $n$, $\dom_n$ denote the indexed set of state variable $\{\v_1,\v_2,\ldots,\v_{\lceil\log n\rceil}\}$.
Let $\state^n_i$ denote the state defined by assigning all the state variables $\dom_n$, s.t.\ their assignments binary encode the natural number $i$, where the index of each variable from $\dom_n$ represents its endianess.
Note: $\state^n_i$ is well defined for $0\leq i\leq 2^{\lceil \log n \rceil} - 1$.

Now, for an arbitrary number $n\in\mathbb{N}$, let $\lotus_n$ denote the factored system (i.e. set of actions) $\{(\state^{n+1}_0, \state^{n+1}_i)\mid 1 \leq i \leq n\}\cup\{(\state^{n+1}_i, \state^{n+1}_0)\mid 1 \leq i \leq n\}$.
The recurrence diameter of the system $\lotus_n$ is $2$, regardless of $n$, since any action sequence that traverses more than $3$ states will traverse $\state^{n+1}_0$ more than once.
Now, let $\probss$ denote the set of states in the largest connected component in the state space of $\lotus_n$, which has $n+1$ states in it.
Since for any two states $\state_i^{n+1},\state_j^{n+1}\in\probss$, there is an action sequence $\as\in\ases{\lotus_n}$ s.t. $\exec{\state_i^{n+1}}{\as} = \state_j^{n+1}$, and since $\cardinality{\probss}=n+1$, then the traversal diameter of $\lotus_n$ is $n$.
Accordingly, and since $2^{\cardinality{\dom(\lotus_n)} - 2} = 2^{\cardinality{\dom_{n+1}} - 2} = 2^{\lceil \log n \rceil - 1} \leq n$, we have that $2^{\cardinality{\dom(\lotus_n)} - 2} \leq \traversalDiam(\lotus_n)$.
The theorem follows from this and since $\recurrenceDiam(\lotus_n) = 2$.
\end{proof}

\begin{myeg}
\label{eg:tdexpltrd}
The state space of $\lotus_3$ is depicted in Figure~\ref{fig:stateDiag3}.
$\recurrenceDiam(\lotus_3) = 2$, and $\traversalDiam(\lotus_3) = 3$.
\end{myeg}
 The fact that $\recurrenceDiam$ can be exponentially smaller than $\traversalDiam$ gives rise to the possibility of substantial improvements to the bounds computed if we use $\recurrenceDiam$ as a base case function for compositional bounding, instead of $\traversalDiam$.

\section{Using the Recurrence Diameter for Compositional Bounding}

Here, we investigate using $\recurrenceDiam$ as a base case function for the compositional algorithm by~\citeauthor{icaps2017}~\citeyear{icaps2017}.
To do that, we firstly devise a method to compute  $\recurrenceDiam$.
For explicitly represented digraphs, the computational complexity of finding the length of the longest path is NP-hard~\cite{pardalos2004note}.
Thus, there is not a known method to compute it with a wort-case running time smaller than a time exponential in the size of the given digraph.
\citeauthor{BiereCCZ99}~\citeyear{BiereCCZ99} suggested the only method to compute $\recurrenceDiam$ of which we are aware.
They encode the question of whether a given number $k$ is $\recurrenceDiam$ of a given transition system as a SAT formula.
$\recurrenceDiam$ is found by querying a SAT-solver for different values of $k$, until the SAT-solver answers positively for one $k$.
The method terminates since $\recurrenceDiam$ cannot be larger than one less the number of states in the given transition system.
The size of their encoding grows linearly in $k^2$.
The encoding of Biere et al. is based on the following theorem, which we restate in our notation.
\newcommand{\encodinga}{\ensuremath{\phi_1}}

{
\renewcommand{\underset}[2]{#2#1.\; }
\begin{mythm}[\citeauthor{BiereCCZ99}~\citeyear{BiereCCZ99}]
For a factored system $\delta$ and a natural number $k$, we have that $\encodinga(\delta,k)$ is true iff $\recurrenceDiam(\delta) < k$, where $\encodinga(\delta,k)$ denotes the conjunction of
\begin{enumerate}
  \item $\underset{(\state,\state')\in\graph(\delta)}{\forall}\graph(\state,\state')$, 
  \item $\underset{\state,\state'\in\uniStates(\delta)}{\forall}$ if $(\state,\state')\not\in\graph(\delta)$, then $\neg\graph(\state,\state')$, and
  \item if $\forall\statea\stateb\dots\state_{k+1}. ({\underset{1\leq i\leq k}{\forall}} \graph(\state_i,\state_{i+1}))$ then $({\underset{1\leq i < j \leq k+1}{\exists}} \state_i=\state_j)$.
\end{enumerate}
\end{mythm}
}

\citeauthor{KroeningS03}~\citeyear{KroeningS03} use sorting networks~\cite{DBLP:books/lib/Knuth98a} to devise another encoding of the above question whose size grows linearly in $k\log^2(k)$.
However, they report that, due to hidden constants, their encoding is only significantly smaller than the encoding by \citeauthor{BiereCCZ99} when $150 < k$.
Since this is typically well beyond recurrence diameters that can be practically computed, we only implement the encoding of \citeauthor{BiereCCZ99}.\footnote{The largest $\recurrenceDiam$ we computed in all our experiments was 93.}

We use an SMT solver to reason about the encoding of $\recurrenceDiam$.
Thus, for decidability as well as efficiency reasons, we would like to obtain an encoding that is quantifier free, in particular, one that fits the theory of quantifier free uninterpreted functions.
Since $\encodinga$ is universally quantified, we reformulate it to its existentially quantified dual.
For predicates $Q$ and $P$ of arity $n$, let $\bigwedge Q(T). \; P(T)$ denote the conjunction of $P(T)$, for all $n$-tuples $T$ where $Q(T)$ holds.
Also, let $\bigvee$ denote the analogous disjunction.
Note: $\bigwedge Q(T). \; P(T)$ is only well defined if $Q$ is true for only a finite set of $n$-tuples.
Also, we do not explicitly bind $Q$ or the tuple $T$ when it is clear from context.

\newcommand{\pathstate}{\ensuremath{y}}
{
\newcommand{\inspace}{\frak{G}}
\renewcommand{\underset}[2]{#2#1.\; }
\begin{myenc}
\label{enc:rdBiere}
For  $\delta$ and $0\leq k$, let $\encodinga'(\delta,k)$ denote the conjunction of the formulae
\begin{enumerate}
  \item $\underset{(\state,\state')\in\graph(\delta)}{\bigwedge}\graph(\state,\state')$, 
  \item $\underset{\{\state,\state'\}\subseteq\uniStates(\delta) \wedge (\state,\state')\not\in\graph(\delta)}{\bigwedge}\neg\graph(\state,\state')$,
  \item ${\underset{1\leq i\leq k}{\bigwedge}} (\graph(\pathstate_i,\pathstate_{i+1}) \wedge \underset{i<j\leq k+1}{\bigwedge}$\\ $ \pathstate_i\neq\pathstate_j)$, and
  \item ${\underset{1\leq i\leq k+1}{\bigwedge}} (\underset{\state\in\uniStates(\delta)}{\bigvee} \pathstate_i=\state)$.
\end{enumerate}
\end{myenc}
}

The SMT formula above is defined over one constant $\state$ for every state in $\uniStates(\delta)$, a set of uninterpreted constants $\{\pathstate_i\mid 1\leq i \leq k+1\}$, one for every state in the simple path of length $k+1$ for which we search, and a function $\graph$ that is true for a pair of constants $(\state, \state')$ iff there is an edge from $\state$ to state $\state'$ in the state space of $\delta$.
\begin{mythm}
\label{thm:rdBiereValid}
$\encodinga'(\delta,k)$ is satisfiable iff $k \leq \recurrenceDiam(\delta)$.
\end{mythm}
\newcommand{\model}{\ensuremath{\mathcal{M}}}
\newcommand{\entail}{\ensuremath{\vDash}}
\begin{proof}[Proof]
\newcommand{\signature}{\ensuremath{\Sigma}}

An SMT formula is defined over a signature $\signature$, which is a finite set of symbols that are either constants, uninterpreted constants, or the standard logical connectives.
A model $\model$ for a signature is a function that maps uninterpreted constants to objects.
A model $\model$ entails a formula $\phi$, denoted $\model \entail \phi$, iff $\phi$ evaluates to true, under the standard interpretation of logical connectives, after each uninterpreted constant $v$ in $\phi$ is substituted by $\model(v)$.
A formula $\phi$ is satisfiable iff $\exists \model. \model \entail \phi$.

\newcommand{\enc}[2]{\ensuremath{enc(#1,#2)}}

\begin{mylem}
\label{lem:encodingasound}
If $\encodinga'(\delta,k)$ is satisfiable, then there is a list of distinct states $[\state_1,\state_2\dots\state_{k+1}]$, such that $(\state_i,\state_{i+1})\in\graph(\delta)$, for $1\leq i\leq k$.
\end{mylem}
\begin{proof}[Proof summary]

Now, since the formula $\encodinga'(\delta,k)$ is satisfiable, there is a model $\model$, s.t. $\model\entail\encodinga'(\delta,k)$.
From the definition of entailment and the third conjunct of $\encodinga'(\delta,k)$, we have that $\graph(\model(\pathstate_i),\model(\pathstate_{i+1}))$ and $\model(\pathstate_i)\neq\model(\pathstate_j)$ hold, for all $1\leq i\leq k$ and 
$i<j\leq k$.
From this, the first, the second, and fourth conjuncts of $\encodinga'(\delta,k)$
,
we have that $(\model(\pathstate_i),\model(\pathstate_{i+1}))\in\graph(\delta)$.
This finishes our proof.
\end{proof}

\begin{mylem}
\label{lem:encodingacomplete}
If there is a list of distinct states $[\state_1,\state_2\dots\state_{k+1}]$, such that $(\state_i,\state_{i+1})\in\graph(\delta)$, for $1\leq i\leq k$, then $\encodinga'(\delta,k)$ is satisfiable.
\end{mylem}
\begin{proof}[Proof summary]
Consider the model $\model$ defined as $\model(\pathstate_i) = \state_i$, if $1\leq i\leq k+1$.
Note that $\model$ is well-defined for the set of uninterpreted constants in $\encodinga'(\delta,k)$, i.e.\ it is well-defined for the set $\{\pathstate_i\mid 1\leq i\leq k+1\}$.
From the assumptions of this lemma and the definition of $\encodinga'(\delta,k)$, we have that $\model\entail\encodinga'(\delta,k)$.
This finishes our proof.
\end{proof}

From the definition of $\recurrenceDiam$, there is a list of actions $\as_k\equiv[\act_1,\act_2\dots\act_k]$ and a state $\state_1\in\uniStates(\delta)$, s.t.\ $\exec{\state_1}{\as_k}$ traverses distinct states, iff $k<\recurrenceDiam(\delta)$.
Also, from the definition of $\graph(\delta)$, for any states $\state$ and $\state'$, there is an action $\act_i\in\delta$ s.t.\ $\state' = \exec{\state}{\act_i}$ iff $(\state,\state')\in\graph(\delta)$.
Accordingly, there is a list of distinct states $[\state_1,\state_2\dots\state_{k+1}]$, s.t.\ $(\state_i,\state_{i+1})\in\graph(\delta)$, for $1\leq i\leq k$, iff $k\leq\recurrenceDiam(\delta)$.
The theorem follows from this and Lemmas~\ref{lem:encodingasound}~and~\ref{lem:encodingacomplete}.
\end{proof}

To use the above encoding to compute $\recurrenceDiam$ of a given system $\delta$, we iteratively query an SMT solver to check for the satisfiability of $\encodinga'(\delta,k)$ for different values of $k$, starting at 1, until the we have an unsatisfiable formula.
The largest $k$ for which the formula is satisfiable is $\recurrenceDiam(\delta)$.

Observe that, to use Encoding~\ref{enc:rdBiere}, one has to build the entire state space as a part of building the encoding, i.e. one has to build the graph $\graph(\delta)$ and include it in the encoding.
In fact, this is true for both methods, the one by \citeauthor{BiereCCZ99} and the one by \citeauthor{KroeningS03}, as they are both specified in terms of explicitly represented transition systems.
This means that the worst-case complexity of computing $\recurrenceDiam$ using either one of those encodings is doubly-exponential.
Indeed, this is the best possible wort-case running time for succinct graphs generally, unless the polynomial hierarchy collapses, since computing $\recurrenceDiam$ is NEXP-hard.

\paragraph{Experimental evaluation} We use Encoding~\ref{enc:rdBiere} as a base case function for the compositional algorithm by~\citeauthor{icaps2017}~\citeyear{icaps2017} instead of $\traversalDiam$, which was used as a base case by~\citeauthor{abdulaziz:2019}~\citeyear{abdulaziz:2019} and led to the tightest bounds of any existing method.
We use Yices 2.6.1~\cite{DBLP:conf/cav/Dutertre14} as the SMT solver to prove the satisfiability or unsatisfiability of the resulting SMT formulae.
We run the bounding algorithm by~\citeauthor{icaps2017}~\citeyear{icaps2017} on standard planning benchmarks (from previous planning competitions and ones we modified), once with $\traversalDiam$ as a base case and a second time with $\recurrenceDiam$ as a base case.
We perform our experiments on a cluster of 2.3GHz Intel Xeon machines with a timeout of 20~minutes and a memory limit of 4GB.
Our experiments show that Encoding~\ref{enc:rdBiere} is not practical for planning problems when used as a base case function for the algorithm by~\citeauthor{icaps2017}~\citeyear{icaps2017}, where bounds are only computed within the timeout for less than 0.1\% of our set of benchmarks.
This is because computing $\recurrenceDiam$ can take time that is exponential in the size of the state space, while computing $\traversalDiam$ can be computed in time that is linear in the state space~\cite{abdulaziz:2019}.

\section{A Compact Encoding of Recurrence Diameter}

We now devise a new encoding that performs better than Encoding~\ref{enc:rdBiere}.
The new encoding exploits the factored representation in a way that is reminiscent to encodings used for SAT-based planning~\cite{kautz:selman:92}.
In particular, our aim is to avoid constructing the state space in an explicit form, whenever possible.
We devise a new encoding that avoids building the state space as a part of the encoding and, effectively, we let the SMT solver build as much of it during its search as needed.
\newcommand{\encodingb}{\ensuremath{\phi_2}}
\newcommand{\encodingba}{\ensuremath{\phi_2'}}
{
\renewcommand{\underset}[2]{#2#1.\; }
\begin{myenc}
\label{enc:rdnew}
For a state $\state$, let $\state_i$ denote the formula $(\underset{\v\in\state}{\bigwedge}\v_i)\wedge(\underset{\overline{\v}\in\state}{\bigwedge}\neg\v_i)$.
For $\delta$ and $0\leq k$, let $\encodingb(\delta,k)$ denote the conjunction of the formulae
\begin{enumerate}
  \item $\underset{1\leq i \leq k}{\bigwedge}\act_i\rightarrow\apre(\act)_i\wedge\aeff(\act)_{i+1}\wedge(\underset{\v\in\dom(\delta)\setminus\dom(\aeff(\act))} {\bigwedge}\v_i\leftrightarrow\v_{i+1})$, 
  \item ${\underset{1\leq i \leq k}{\bigwedge}}\underset{\act\in\delta}{\bigvee}\act_i$, and
  \item $\underset{1\leq i < j \leq k + 1}{\bigwedge}\underset{v\in\dom(\delta)}{\bigvee}\v_i\neq\v_j$.
\end{enumerate}
\end{myenc}
}

Briefly, the encoding above states that $k$ is not $\recurrenceDiam$ if there is a sequence of $k$ actions that traverses only distinct states if executed at some valid state.
In more detail, the following are the intuitive meanings of uninterpreted constants in the above formulae: (i) $\act_i$, for all $1\leq i\leq k$ and $\act\in\delta$, is a Boolean variable that represents whether action $\act$ is executed at state $i$, and (ii) $\v_i$, for all $1\leq i\leq k + 1$ and $\v\in\delta$, represents the truth value of state variable $\v$ at state $i$.\footnote{We note that this encoding can easily be formulated as a propositional formula in conjunctive normal format.}

There are three main conjuncts in the encoding.
The first conjunct formalises the fact that, if an action is executed at state $i$, then all of its preconditions hold at state $i$, all of its effects hold at state $i+1$, and all the variables that are not in the effects will continue to have the same value at state $i+1$ as they did at state $i$ (i.e. the frame axiom).
The second conjunct states that at least one action must execute at state $i$.
The third conjunct states that all states are pairwise distinct by stating that for every two states, at least one variable has a different truth value in both states.

\begin{mythm}
\label{thm:rdnewValid}
$\encodingb(\delta,k)$ is satisfiable iff $k \leq \recurrenceDiam(\delta)$.
\end{mythm}
\begin{proof}
Firstly, let $\encodingba(\delta,k)$ denote \[\encodingb(\delta,k) \wedge
\bigwedge{1\leq i \leq k}.{\bigwedge \act,\act'\in\delta\wedge \act\neq\act'.}\neg\act_i \vee \neg\act'_i.\]
\begin{mylem}
\label{lem:rdnewValid}
$\encodingb(\delta,k)$ is satisfiable iff $\encodingba(\delta,k)$ is satisfiable.
\end{mylem}
\begin{proof}[Proof summary]
$\Rightarrow$ Since $\encodingb(\delta,k)$ is satisfiable, then there is a model $\model$, s.t.\ $\model\entail\encodingb(\delta,k)$.
Note that $\model$ might not entail $\encodingba(\delta,k)$ because $\encodingba(\delta,k)$ has the extra conjunct that only one action is enabled in every step, i.e.\ at step $i$, if $\model\entail\act_i$ and $\model\entail\act'_i$, then $\act=\act'$.
Nonetheless, conjunct (i) of Encoding 2 necessitates that in order for an action to be enabled in a step, all variables that are not in its effect are left unchanged in the next step. 
Accordingly, all actions enabled at a step affect the same state variables and assign all of those variables to the same value.
Thus, we can construct a model that entails $\encodingba(\delta,k)$ by leaving only one action from the set of enabled actions at every step, and disabling the rest.
We formalise that as follows.
For every $0\leq i\leq k$, let $\Pi_i = \{\act_i \mid \act \in \delta \wedge \model\entail\act_i\}$, i.e.\ the set of actions enabled in step $i$.
Let $\choice$ be the choice function, i.e.\ the function that given a set, returns an element from that set if it is not empty, and that is otherwise undefined.
The model that entails $\encodingba(\delta,k)$ is $\model'$, defined as follows.
\[
  \model'(c) =
  \begin{cases}
     \bot, & \text{if } p\in\Pi_i\text{ and } \choice(\Pi_i)\neq p\\
     \model(c), & \text{otherwise}.\\
  \end{cases}
\]
$\Leftarrow$ Since $\encodingba(\delta,k)$ is the same as $\encodingb(\delta,k)$ conjoined with another formula, any model for $\encodingba(\delta,k)$ is a model for $\encodingb(\delta,k)$.
\end{proof}
The theorem follows from Lemma~\ref{lem:rdnewValid}, Theorem~\ref{thm:rdBiereValid}, and since $\encodinga'(\delta,k) \leftrightarrow \encodingba(\delta,k)$.
The latter fact follows by an induction on $k$, and from the definition of $\dom(\delta)$ and $\graph(\delta)$.
\end{proof}

\paragraph{Experimental evaluation} We experimentally test the new encoding as a base case function for the algorithm by~\citeauthor{icaps2017}~\citeyear{icaps2017}.
Columns 2 and 3 of Table~\ref{table:boundsDoms} show some data on the bounds computed with both, Encoding~\ref{enc:rdnew} (i.e.\ $\recurrenceDiam$) and $\traversalDiam$, as base case functions.
We note two observations.
Firstly, many more planning problems are successfully bounded within the timeout when Encoding~\ref{enc:rdnew} is used to compute $\recurrenceDiam$ compared to using Encoding~\ref{enc:rdBiere}.
Encoding~\ref{enc:rdnew} performs much better than Encoding~\ref{enc:rdBiere} in practice since our new encoding is represented in terms of the factored representation of the system, while Encoding~\ref{enc:rdBiere} represents the system as an explicitly represented state space.
This leads to exponentially smaller formulae: Encoding~\ref{enc:rdnew} grows quadratically with the size of the given factored system, while Encoding~\ref{enc:rdBiere} grows quadratically in the size of the state space, which can be exponentially larger than the given factored system.
Indeed, Encoding~\ref{enc:rdnew} delegates the construction of the explicit state space to the SMT solver, which would effectively construct the state space during its search, but lazily.
This is clearly better than constructing the state space a priori when the formula is satisfiable (i.e. when $k\leq\recurrenceDiam$) as the SMT solver only needs to find a simple path of length $k+1$.
The SMT solver does this without necessarily traversing the entire state space due to its search heuristics.
When the formula is unsatisfiable, the SMT solver has to perform an exhaustive search to produce a proof of unsatisfiability, which is equivalent to constructing the entire state space explicitly.
Since all queries to the SMT solver, except for the last one, are satisfiable, Encoding~\ref{enc:rdnew} is more practically efficient than using Encoding~\ref{enc:rdBiere}.
However, it is worth noting that Encodings~\ref{enc:rdBiere}~and~\ref{enc:rdnew} worst-case running times grow doubly-exponentially in the size of the given factored system.

Secondly, when $\recurrenceDiam$ is the base case function the bounds computed are much tighter than those computed when $\traversalDiam$ as a base case function.
This agrees with the theoretical prediction of Theorem~\ref{lem:rdexplttd}.
This is shown clearly in Figure~\ref{fig:boundcompareTDvsRD} and in Table~\ref{table:boundsDoms}. 
In particular, in the domains TPP, ParcPrinter, NoMystery, Logistics, OpenStacks, Woodworking, Satellites, Scanalyzer, Hyp and NewOpen (a compiled Qualitative Preference rovers domain), we have between two orders of magnitude and 50\% smaller bounds when $\recurrenceDiam$ is used as a base case function compared to $\traversalDiam$.
Also, the domain Visitall has twice as many problems whose bounds are less than $10^9$ when $\recurrenceDiam$ is used instead of $\traversalDiam$.
Also, specially interesting domains are Floortile and BlocksWorld, where the recurrence diameter of some of the smaller instances is successfully computed to be less than 10, but whose bounds using $\traversalDiam$ are more than $10^9$.
In contrast, equal bounds are found using $\recurrenceDiam$ and $\traversalDiam$ as base case functions in Zeno.

 \section{Further Experiments}
\begin{figure*}
\centering
\begin{minipage}[b]{0.32\textwidth}
\centering
       \includegraphics[width=1\textwidth,height=0.8\textwidth]{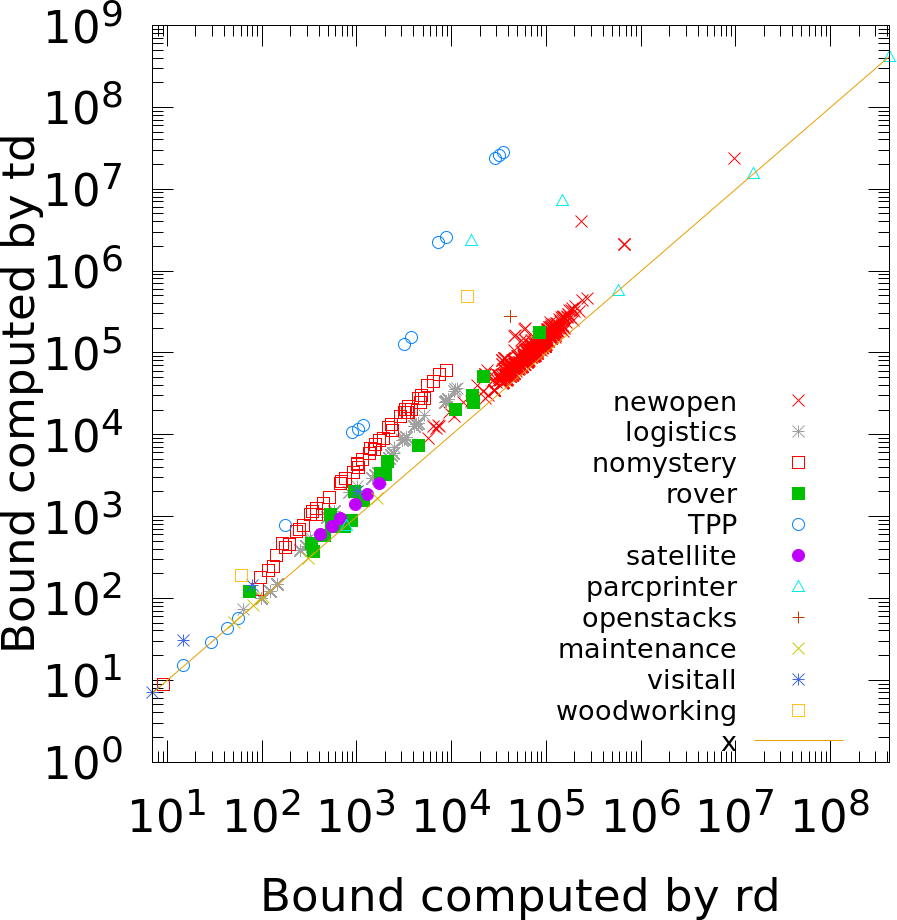}
\end{minipage}
\begin{minipage}[b]{0.32\textwidth}
\centering
        \includegraphics[width=1\textwidth,height=0.8\textwidth]{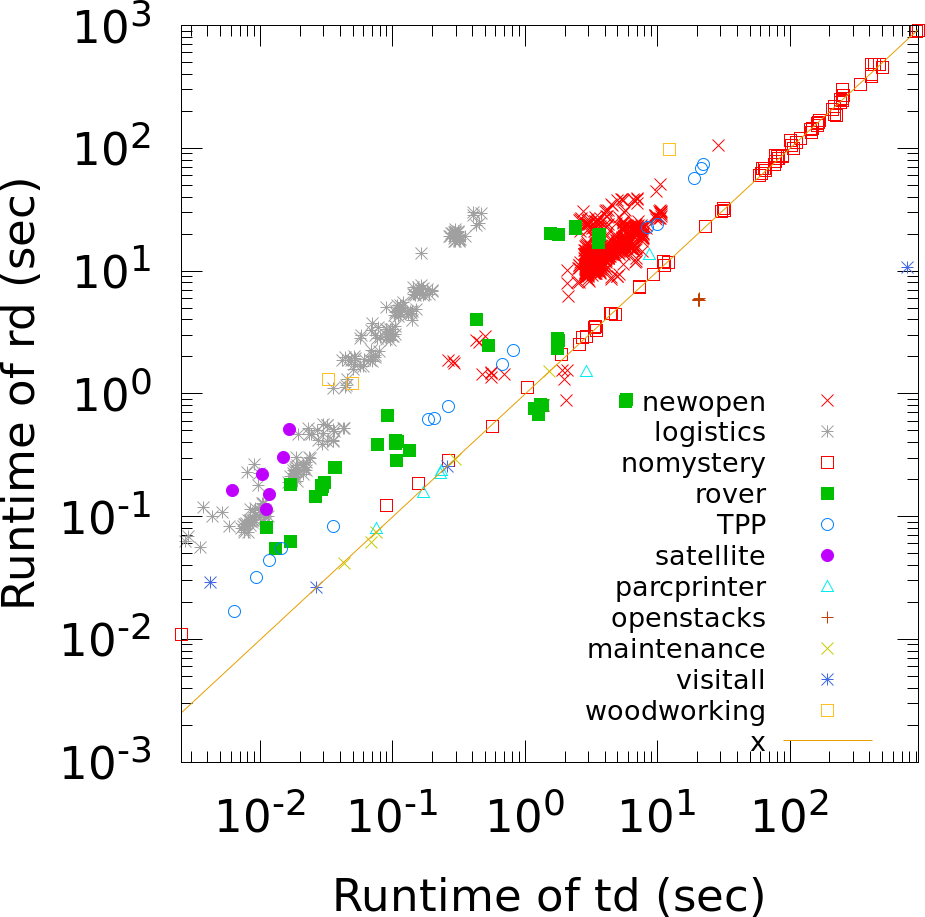}
\end{minipage}
\begin{minipage}[b]{0.32\textwidth}
        \includegraphics[width=1\textwidth,height=0.8\textwidth]{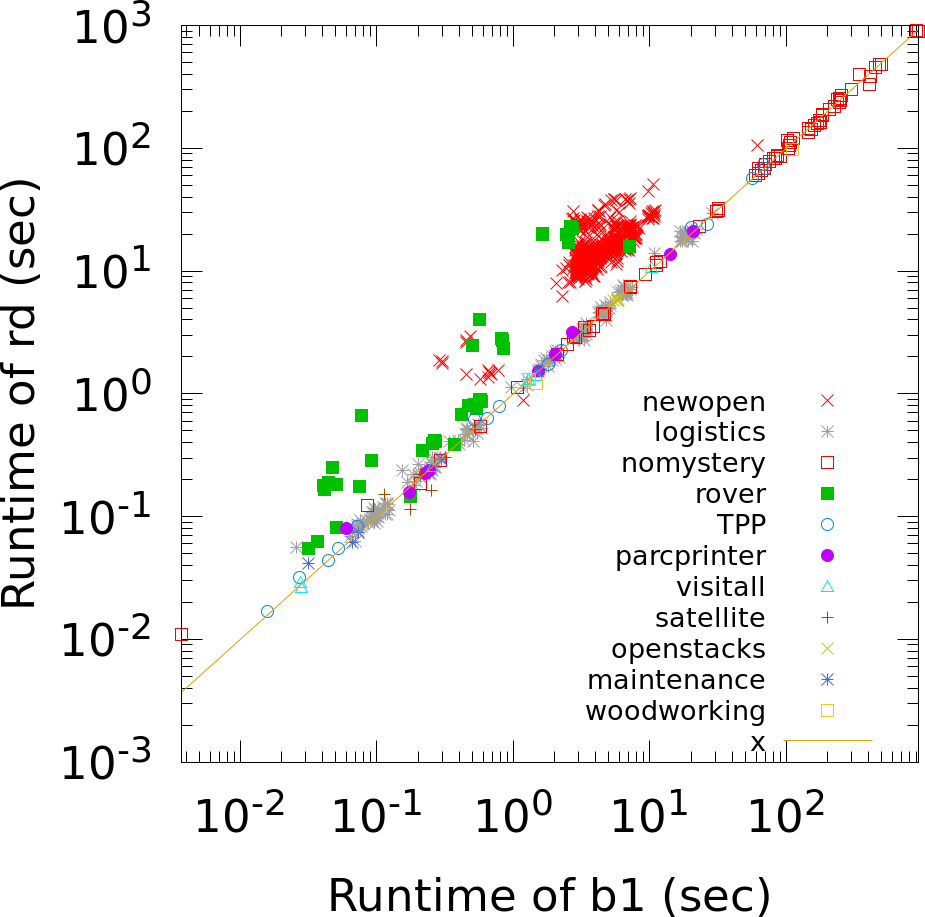}
\end{minipage}
\begin{minipage}[b]{0.32\textwidth}
        \includegraphics[width=1\textwidth,height=0.8\textwidth]{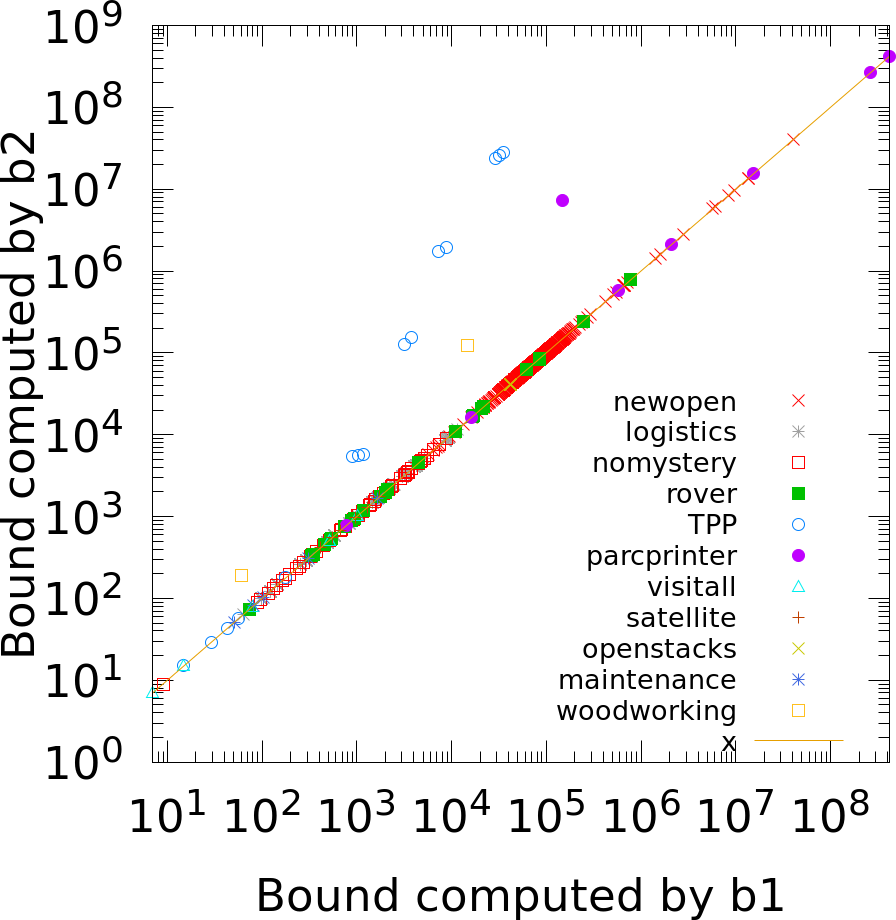}
\end{minipage}
\begin{minipage}[b]{0.32\textwidth}
\centering
        \includegraphics[width=1\textwidth,height=0.8\textwidth]{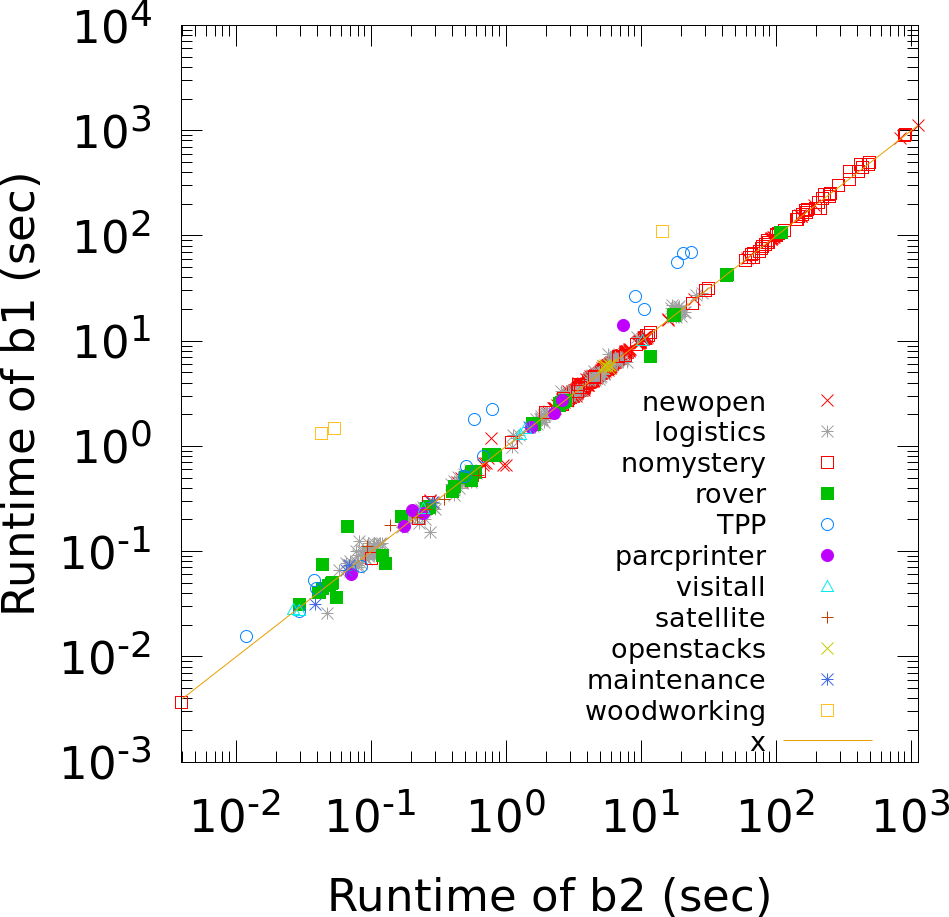}
\end{minipage}
\begin{minipage}[b]{0.32\textwidth}
\centering
       \includegraphics[width=1\textwidth,height=0.8\textwidth]{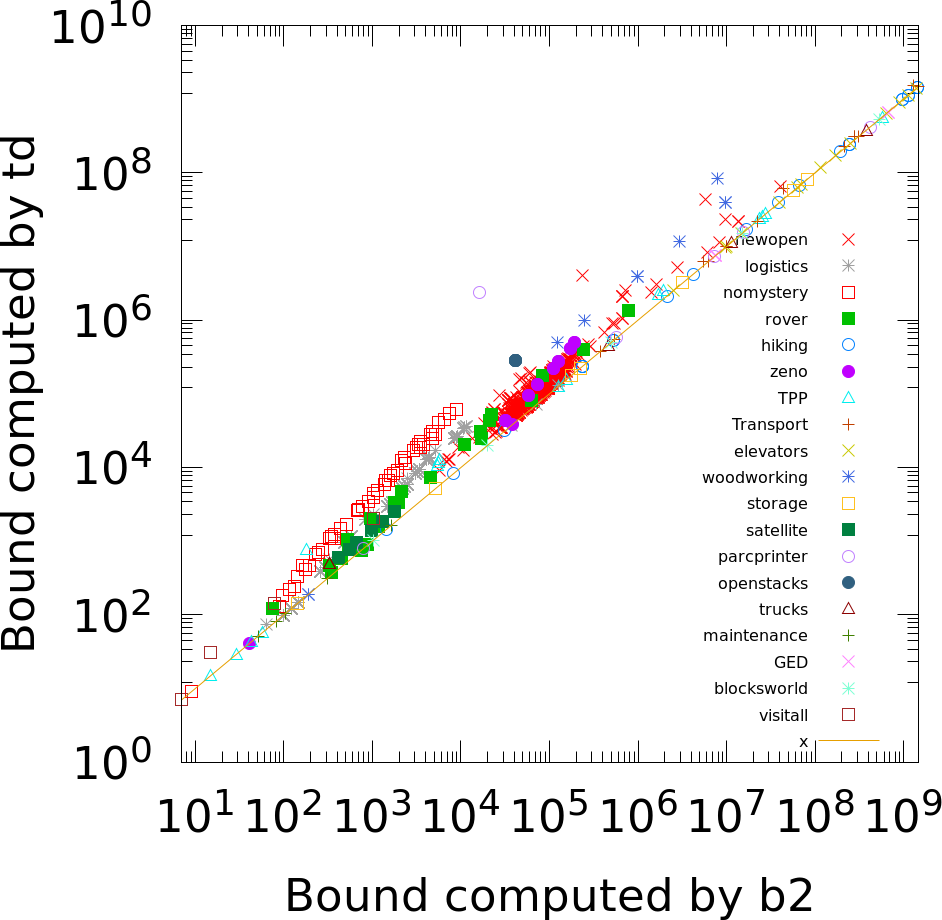}
\end{minipage}
\begin{minipage}[b]{0.32\textwidth}
\centering
\includegraphics[width=1\textwidth,height=0.8\textwidth]{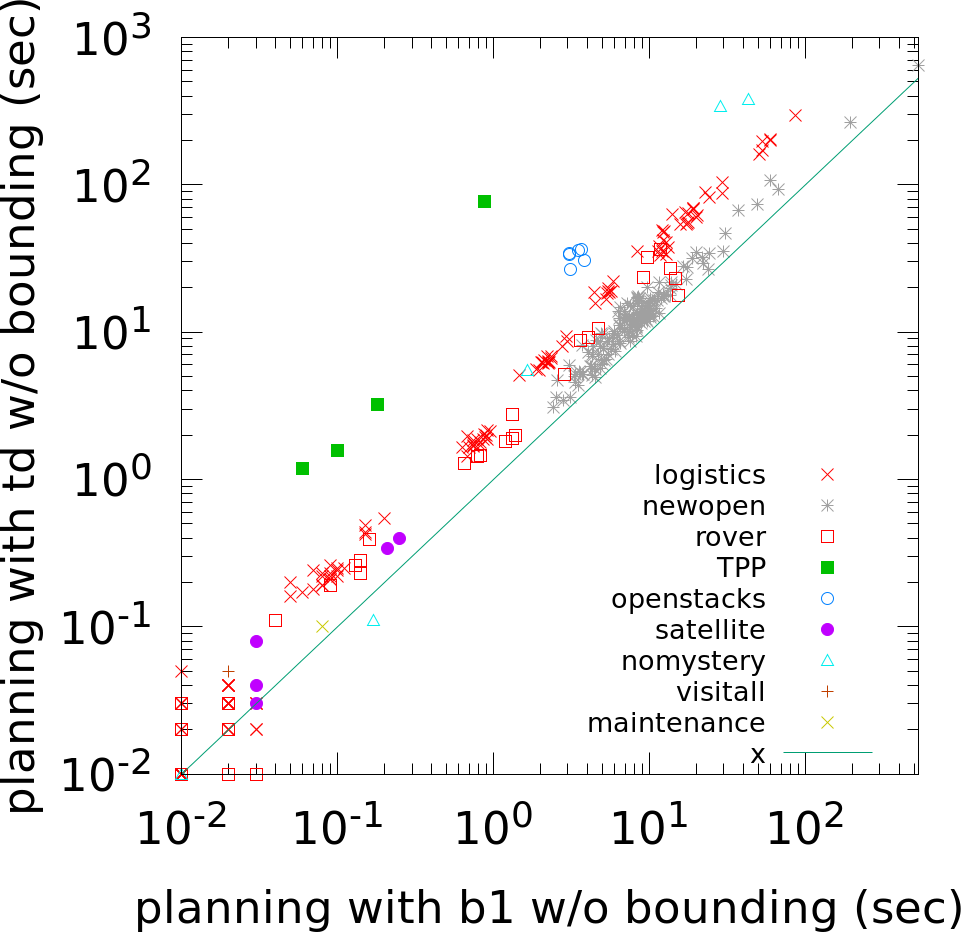}
\end{minipage}
\begin{minipage}[b]{0.32\textwidth}
\centering
\includegraphics[width=1\textwidth,height=0.8\textwidth]{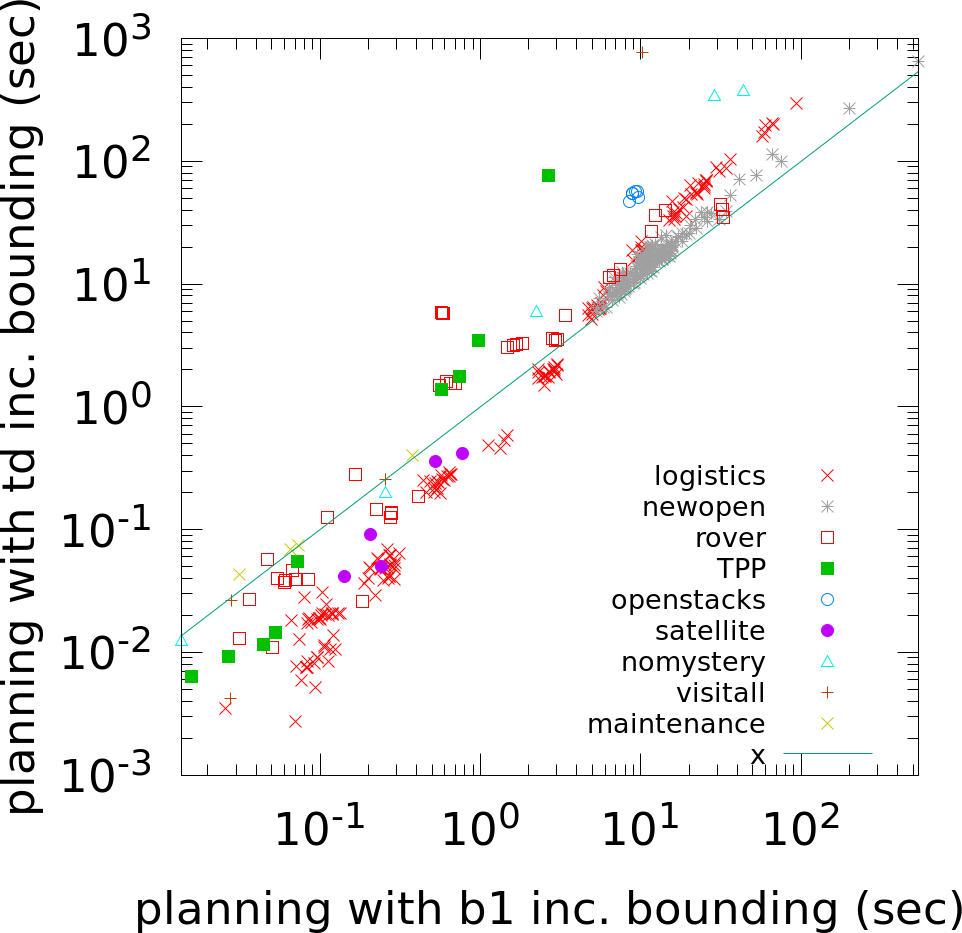}
\end{minipage}
\begin{minipage}[b]{0.32\textwidth}
\centering
\includegraphics[width=1\textwidth,height=0.8\textwidth]{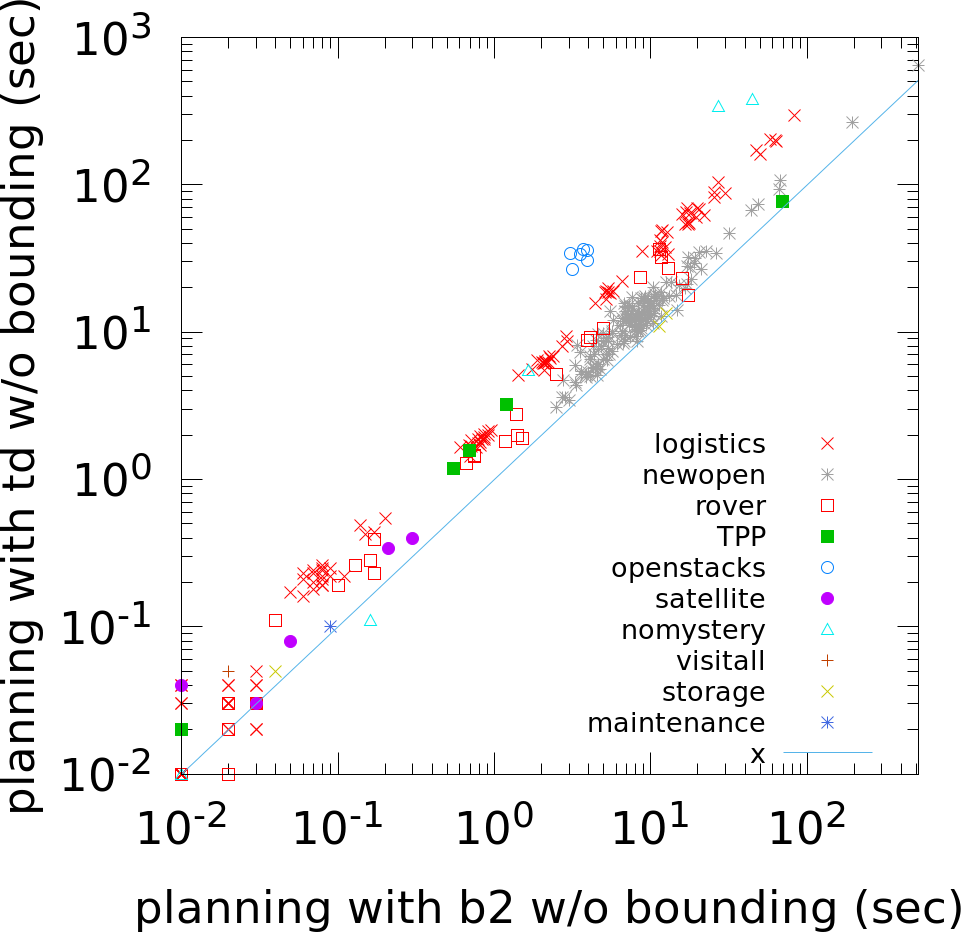}
\end{minipage}
\begin{minipage}[b]{0.32\textwidth}
\centering
\includegraphics[width=1\textwidth,height=0.8\textwidth]{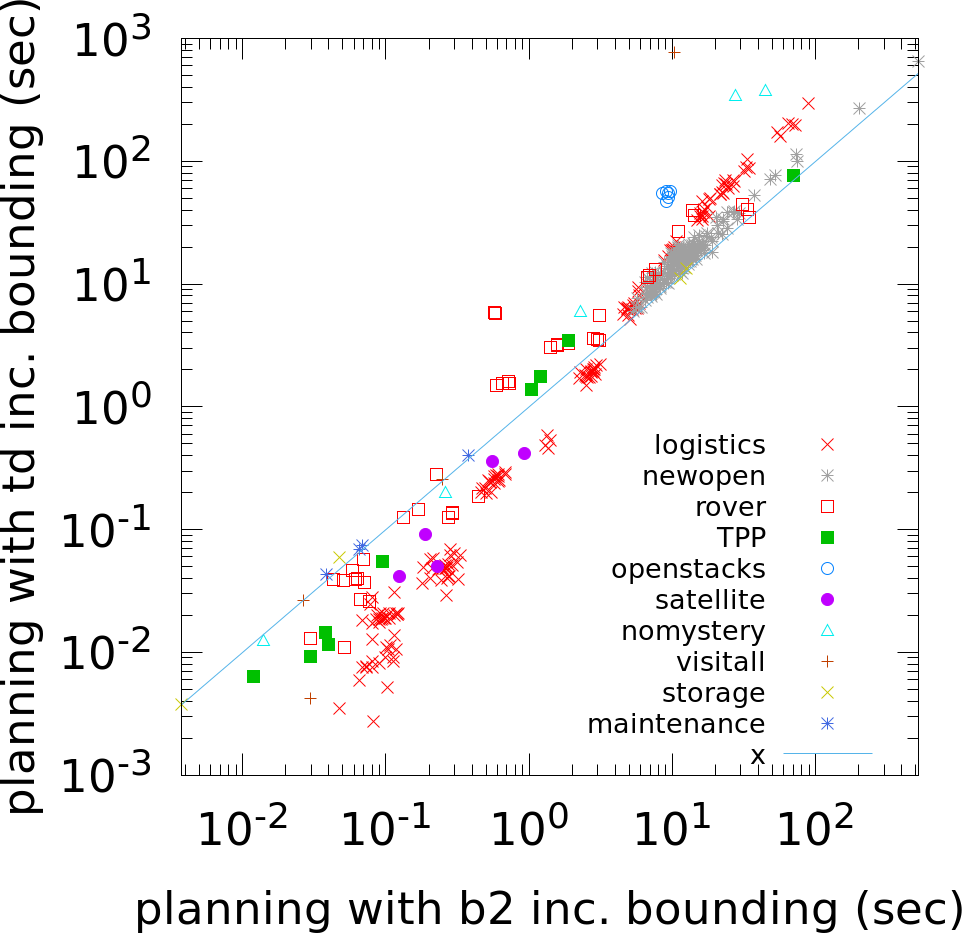}
\end{minipage}
\begin{minipage}[b]{0.32\textwidth}
\centering
\includegraphics[width=1\textwidth,height=0.8\textwidth]{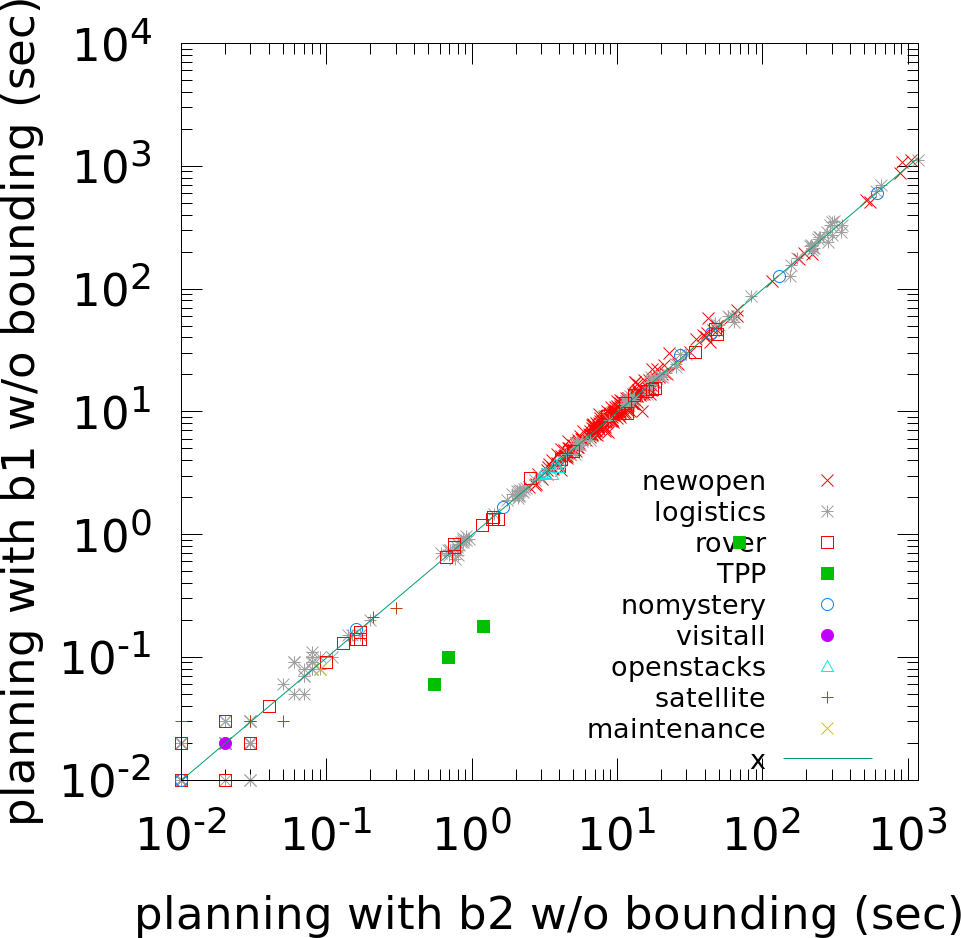}
\end{minipage}
\begin{minipage}[b]{0.32\textwidth}
\centering
\includegraphics[width=1\textwidth,height=0.8\textwidth]{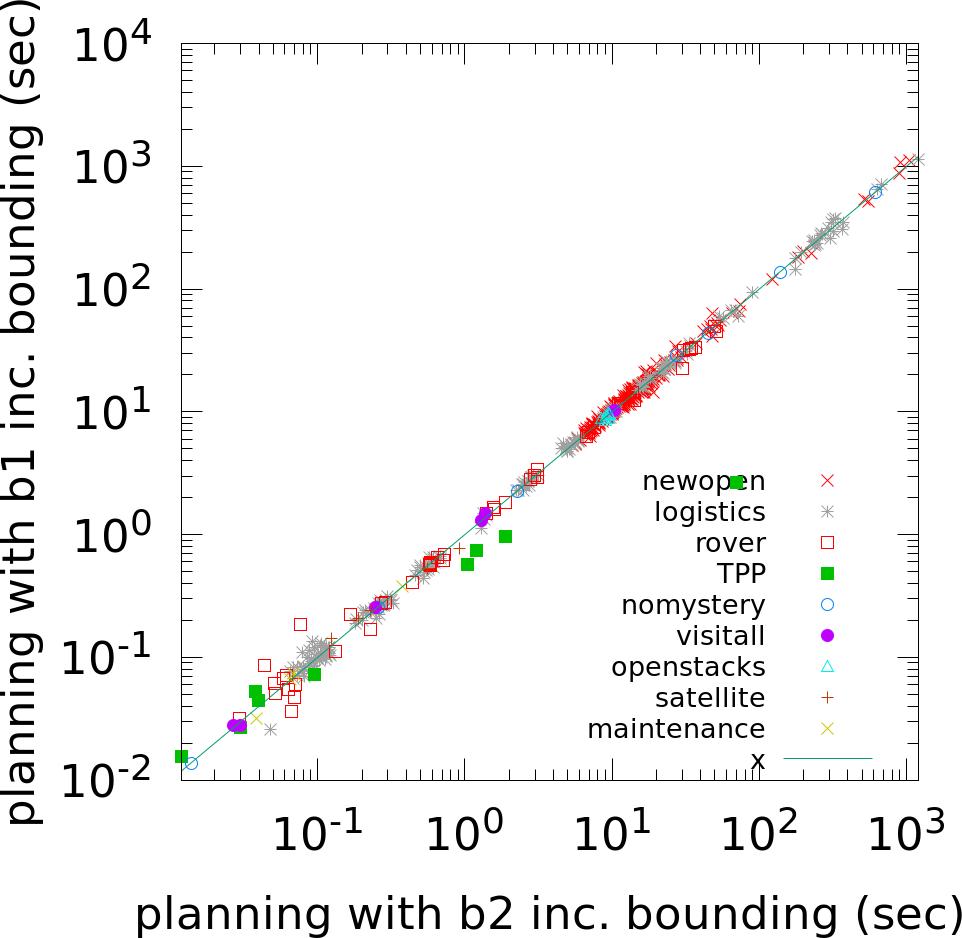}
\end{minipage}
     \caption{\label{fig:boundcompareTDvsRD}Different scatter-plots comparing the different bounding algorithms in terms of the quality of their bounds, the needed bound computation time, and the planning time using those bounds as a horizon for {\sc Mp}.}
\end{figure*}
{
\setlength\tabcolsep{1pt}
\def\arraystretch{0.9}
\begin{table*}[h]
\centering
\setlength{\tabcolsep}{3.5pt}
    \begin{tabularx}{0.99\textwidth}{ c | c c c c c | c c c c c | c c c c c | c c c c c }
    \hline
          \multicolumn{1}{c}{\scriptsize }       &\multicolumn{5}{c}{\scriptsize $\traversalDiam$}       &\multicolumn{5}{c}{\scriptsize $\recurrenceDiam$}       &\multicolumn{5}{c}{\scriptsize $\basecasefun_1$}       &\multicolumn{5}{c}{\scriptsize $\basecasefun_2$}\\ 
    \hline
{\scriptsize Domain}                      &{\scriptsize Min.} &{\scriptsize Max.} &{\scriptsize Avg.} &{\scriptsize \#Bnd.} &{\scriptsize \#Sol.} &{\scriptsize Min.} &{\scriptsize Max.} &{\scriptsize Avg.} &{\scriptsize \#Bnd.} &{\scriptsize \#Sol.} &{\scriptsize Min.} &{\scriptsize Max.} &{\scriptsize Avg.} &{\scriptsize \#Bnd.} &{\scriptsize \#Sol.} &{\scriptsize Min.} &{\scriptsize Max.} &{\scriptsize Avg.} &{\scriptsize \#Bnd.} &{\scriptsize \#Sol.}\\
    \hline
{\scriptsize newopen (1440) }       &{\scriptsize 3e3} &{\scriptsize 7e7  } &{\scriptsize 4e5  } &{\scriptsize 848} &{\scriptsize 157}       &{\scriptsize 2e3} &{\scriptsize 1e7  } &{\scriptsize 9e4  } &{\scriptsize 592} &{\scriptsize ----}       &{\scriptsize 2e3} &{\scriptsize 4e7  } &{\scriptsize 2e5  } &{\scriptsize 845} &{\scriptsize 218}       &{\scriptsize 2e3} &{\scriptsize 4e7  } &{\scriptsize 2e5  } &{\scriptsize 844} &{\scriptsize 219}\\  
{\scriptsize logistics (407) }       &{\scriptsize 7e1} &{\scriptsize 4e6  } &{\scriptsize 4e5  } &{\scriptsize 406} &{\scriptsize 170}       &{\scriptsize 6e1} &{\scriptsize 1e4  } &{\scriptsize 3e3  } &{\scriptsize 194} &{\scriptsize ----}       &{\scriptsize 6e1} &{\scriptsize 1e4  } &{\scriptsize 2e3  } &{\scriptsize 193} &{\scriptsize 192}       &{\scriptsize 6e1} &{\scriptsize 1e5  } &{\scriptsize 4e3  } &{\scriptsize 200} &{\scriptsize 195}\\  
{\scriptsize elevators (210) }       &{\scriptsize 3e6} &{\scriptsize 1e9  } &{\scriptsize 3e8  } &{\scriptsize 14} &{\scriptsize ----}       &{\scriptsize ----} &{\scriptsize ----  } &{\scriptsize ----  } &{\scriptsize ----} &{\scriptsize ----}       &{\scriptsize ----} &{\scriptsize ----  } &{\scriptsize ----  } &{\scriptsize ----} &{\scriptsize ----}       &{\scriptsize 3e6} &{\scriptsize 1e9  } &{\scriptsize 3e8  } &{\scriptsize 14} &{\scriptsize ----}\\  
{\scriptsize rover (141) }       &{\scriptsize 1e2} &{\scriptsize 1e6  } &{\scriptsize 1e5  } &{\scriptsize 51} &{\scriptsize 42}       &{\scriptsize 7e1} &{\scriptsize 9e4  } &{\scriptsize 1e4  } &{\scriptsize 38} &{\scriptsize ----}       &{\scriptsize 7e1} &{\scriptsize 8e5  } &{\scriptsize 7e4  } &{\scriptsize 52} &{\scriptsize 46}       &{\scriptsize 7e1} &{\scriptsize 8e5  } &{\scriptsize 7e4  } &{\scriptsize 52} &{\scriptsize 46}\\  
{\scriptsize nomystery (124) }       &{\scriptsize 9e0} &{\scriptsize 6e4  } &{\scriptsize 1e4  } &{\scriptsize 70} &{\scriptsize 6}       &{\scriptsize 9e0} &{\scriptsize 9e3  } &{\scriptsize 2e3  } &{\scriptsize 70} &{\scriptsize ----}       &{\scriptsize 9e0} &{\scriptsize 9e3  } &{\scriptsize 2e3  } &{\scriptsize 70} &{\scriptsize 8}       &{\scriptsize 9e0} &{\scriptsize 9e3  } &{\scriptsize 2e3  } &{\scriptsize 70} &{\scriptsize 7}\\  
{\scriptsize zeno (50) }       &{\scriptsize 4e1} &{\scriptsize 5e5  } &{\scriptsize 8e4  } &{\scriptsize 50} &{\scriptsize 31}       &{\scriptsize 4e1} &{\scriptsize 4e1  } &{\scriptsize 4e1  } &{\scriptsize 1} &{\scriptsize ----}       &{\scriptsize 4e1} &{\scriptsize 4e1  } &{\scriptsize 4e1  } &{\scriptsize 1} &{\scriptsize 1}       &{\scriptsize 4e1} &{\scriptsize 2e5  } &{\scriptsize 1e5  } &{\scriptsize 16} &{\scriptsize 1}\\  
{\scriptsize hiking (40) }       &{\scriptsize 1e3} &{\scriptsize 1e9  } &{\scriptsize 4e8  } &{\scriptsize 22} &{\scriptsize 1}       &{\scriptsize ----} &{\scriptsize ----  } &{\scriptsize ----  } &{\scriptsize ----} &{\scriptsize ----}       &{\scriptsize ----} &{\scriptsize ----  } &{\scriptsize ----  } &{\scriptsize ----} &{\scriptsize ----}       &{\scriptsize 1e3} &{\scriptsize 1e9  } &{\scriptsize 4e8  } &{\scriptsize 22} &{\scriptsize 1}\\  
{\scriptsize TPP (89) }       &{\scriptsize 2e1} &{\scriptsize 6e8  } &{\scriptsize 4e7  } &{\scriptsize 16} &{\scriptsize 9}       &{\scriptsize 2e1} &{\scriptsize 4e4  } &{\scriptsize 8e3  } &{\scriptsize 15} &{\scriptsize ----}       &{\scriptsize 2e1} &{\scriptsize 4e4  } &{\scriptsize 8e3  } &{\scriptsize 15} &{\scriptsize 14}       &{\scriptsize 2e1} &{\scriptsize 6e8  } &{\scriptsize 4e7  } &{\scriptsize 16} &{\scriptsize 9}\\  
{\scriptsize Transport (203) }       &{\scriptsize 4e5} &{\scriptsize 2e9  } &{\scriptsize 3e8  } &{\scriptsize 14} &{\scriptsize ----}       &{\scriptsize ----} &{\scriptsize ----  } &{\scriptsize ----  } &{\scriptsize ----} &{\scriptsize ----}       &{\scriptsize ----} &{\scriptsize ----  } &{\scriptsize ----  } &{\scriptsize ----} &{\scriptsize ----}       &{\scriptsize 4e5} &{\scriptsize 1e9  } &{\scriptsize 3e8  } &{\scriptsize 14} &{\scriptsize ----}\\  
{\scriptsize GED (97) }       &{\scriptsize 7e6} &{\scriptsize 7e8  } &{\scriptsize 3e8  } &{\scriptsize 5} &{\scriptsize ----}       &{\scriptsize ----} &{\scriptsize ----  } &{\scriptsize ----  } &{\scriptsize ----} &{\scriptsize ----}       &{\scriptsize ----} &{\scriptsize ----  } &{\scriptsize ----  } &{\scriptsize ----} &{\scriptsize ----}       &{\scriptsize 7e6} &{\scriptsize 7e8  } &{\scriptsize 3e8  } &{\scriptsize 5} &{\scriptsize ----}\\  
{\scriptsize woodworking (60) }       &{\scriptsize 2e2} &{\scriptsize 8e7  } &{\scriptsize 2e7  } &{\scriptsize 10} &{\scriptsize ----}       &{\scriptsize 6e1} &{\scriptsize 1e4  } &{\scriptsize 5e3  } &{\scriptsize 3} &{\scriptsize ----}       &{\scriptsize 6e1} &{\scriptsize 1e4  } &{\scriptsize 5e3  } &{\scriptsize 3} &{\scriptsize 1}       &{\scriptsize 2e2} &{\scriptsize 1e7  } &{\scriptsize 3e6  } &{\scriptsize 10} &{\scriptsize 1}\\  
{\scriptsize visitall (70) }       &{\scriptsize 7e0} &{\scriptsize 2e3  } &{\scriptsize 6e2  } &{\scriptsize 4} &{\scriptsize 4}       &{\scriptsize 7e0} &{\scriptsize 7e4  } &{\scriptsize 1e4  } &{\scriptsize 7} &{\scriptsize ----}       &{\scriptsize 7e0} &{\scriptsize 1e3  } &{\scriptsize 4e2  } &{\scriptsize 6} &{\scriptsize 6}       &{\scriptsize 7e0} &{\scriptsize 1e3  } &{\scriptsize 4e2  } &{\scriptsize 6} &{\scriptsize 6}\\  
{\scriptsize openstacks (111) }       &{\scriptsize 3e5} &{\scriptsize 3e5  } &{\scriptsize 3e5  } &{\scriptsize 6} &{\scriptsize 6}       &{\scriptsize 4e4} &{\scriptsize 4e4  } &{\scriptsize 4e4  } &{\scriptsize 6} &{\scriptsize ----}       &{\scriptsize 4e4} &{\scriptsize 4e4  } &{\scriptsize 4e4  } &{\scriptsize 6} &{\scriptsize 6}       &{\scriptsize 4e4} &{\scriptsize 4e4  } &{\scriptsize 4e4  } &{\scriptsize 6} &{\scriptsize 6}\\  
{\scriptsize satellite (10) }       &{\scriptsize 6e2} &{\scriptsize 6e3  } &{\scriptsize 3e3  } &{\scriptsize 10} &{\scriptsize 9}       &{\scriptsize 4e2} &{\scriptsize 2e3  } &{\scriptsize 9e2  } &{\scriptsize 6} &{\scriptsize ----}       &{\scriptsize 4e2} &{\scriptsize 2e3  } &{\scriptsize 9e2  } &{\scriptsize 6} &{\scriptsize 5}       &{\scriptsize 4e2} &{\scriptsize 2e3  } &{\scriptsize 9e2  } &{\scriptsize 6} &{\scriptsize 5}\\  
{\scriptsize scanalyzer (60) }       &{\scriptsize 4e3} &{\scriptsize 4e3  } &{\scriptsize 4e3  } &{\scriptsize 1} &{\scriptsize 1}       &{\scriptsize 6e1} &{\scriptsize 6e1  } &{\scriptsize 6e1  } &{\scriptsize 1} &{\scriptsize ----}       &{\scriptsize 6e1} &{\scriptsize 6e1  } &{\scriptsize 6e1  } &{\scriptsize 1} &{\scriptsize 1}       &{\scriptsize 4e3} &{\scriptsize 4e3  } &{\scriptsize 4e3  } &{\scriptsize 1} &{\scriptsize 1}\\  
{\scriptsize storage (30) }       &{\scriptsize 1e2} &{\scriptsize 8e7  } &{\scriptsize 2e7  } &{\scriptsize 7} &{\scriptsize 4}       &{\scriptsize 6e0} &{\scriptsize 6e0  } &{\scriptsize 6e0  } &{\scriptsize 1} &{\scriptsize ----}       &{\scriptsize 6e0} &{\scriptsize 6e0  } &{\scriptsize 6e0  } &{\scriptsize 1} &{\scriptsize 1}       &{\scriptsize 1e2} &{\scriptsize 8e7  } &{\scriptsize 2e7  } &{\scriptsize 7} &{\scriptsize 4}\\  
{\scriptsize trucks (33) }       &{\scriptsize 5e2} &{\scriptsize 4e8  } &{\scriptsize 8e7  } &{\scriptsize 5} &{\scriptsize 2}       &{\scriptsize 3e2} &{\scriptsize 3e2  } &{\scriptsize 3e2  } &{\scriptsize 2} &{\scriptsize ----}       &{\scriptsize 3e2} &{\scriptsize 3e2  } &{\scriptsize 3e2  } &{\scriptsize 2} &{\scriptsize 2}       &{\scriptsize 3e2} &{\scriptsize 4e8  } &{\scriptsize 8e7  } &{\scriptsize 5} &{\scriptsize 2}\\  
{\scriptsize parcprinter (40) }       &{\scriptsize 8e2} &{\scriptsize 4e8  } &{\scriptsize 7e7  } &{\scriptsize 6} &{\scriptsize 1}       &{\scriptsize 8e2} &{\scriptsize 4e8  } &{\scriptsize 8e7  } &{\scriptsize 9} &{\scriptsize ----}       &{\scriptsize 8e2} &{\scriptsize 4e8  } &{\scriptsize 8e7  } &{\scriptsize 9} &{\scriptsize 3}       &{\scriptsize 8e2} &{\scriptsize 4e8  } &{\scriptsize 9e7  } &{\scriptsize 8} &{\scriptsize 2}\\  
{\scriptsize maintenance (5) }       &{\scriptsize 5e1} &{\scriptsize 2e3  } &{\scriptsize 4e2  } &{\scriptsize 5} &{\scriptsize 4}       &{\scriptsize 5e1} &{\scriptsize 2e3  } &{\scriptsize 4e2  } &{\scriptsize 5} &{\scriptsize ----}       &{\scriptsize 5e1} &{\scriptsize 2e3  } &{\scriptsize 4e2  } &{\scriptsize 5} &{\scriptsize 4}       &{\scriptsize 5e1} &{\scriptsize 2e3  } &{\scriptsize 4e2  } &{\scriptsize 5} &{\scriptsize 4}\\  
{\scriptsize blocksworld (10) }       &{\scriptsize 1e3} &{\scriptsize 5e8  } &{\scriptsize 1e8  } &{\scriptsize 5} &{\scriptsize 2}       &{\scriptsize 8e0} &{\scriptsize 8e0  } &{\scriptsize 8e0  } &{\scriptsize 1} &{\scriptsize ----}       &{\scriptsize 8e0} &{\scriptsize 8e0  } &{\scriptsize 8e0  } &{\scriptsize 1} &{\scriptsize 1}       &{\scriptsize 1e3} &{\scriptsize 5e8  } &{\scriptsize 1e8  } &{\scriptsize 5} &{\scriptsize 2}\\  
         \hline
    \end{tabularx}
    \caption{\label{table:boundsDoms}
Column 1: the domain name and the number of instances in it.
Column 2: when using $\traversalDiam$ as a base case function: the minimum, maximum bound, the average bound computed, number of instances bounded (below $10^9$), and the number of instances solved using Madagascar with the bound as the horizon.
Column 3, 4, and 5: similar to column 2, but when $\recurrenceDiam$, $\basecasefun_1$, and $\basecasefun_2$, respectively, are used as base case functions. 
}
\end{table*}
}

Note that, although Encoding~\ref{enc:rdnew} is more efficient than Encoding~\ref{enc:rdBiere}, the number of problems that were successfully bounded is less when using $\recurrenceDiam$ as a base case function compared to $\traversalDiam$, since $\recurrenceDiam$ can take exponentially longer to compute than $\traversalDiam$.
Figure~\ref{fig:boundcompareTDvsRD} shows that running time difference for problems successfully bounded using both base case functions.
Also Table~\ref{table:boundsDoms} shows this in terms of the numbers of problems bounded within 20 minutes, when using $\recurrenceDiam$ vs.\ $\traversalDiam$.

One thing we observe during our experiments is that many of the base case systems have traversal diameters that are 1 or 2.
We exploit that to improve the running time by devising a base case function that will only invoke the expensive computation of $\recurrenceDiam$ in case $\traversalDiam$ is greater than 2.
\renewcommand{\thennew}{\mbox{\upshape \textsf{thn}}}
\renewcommand{\elsenew}{\mbox{\upshape \textsf{els}}}
\begin{mydef}
\[   \basecasefun_1(\delta)=
    \begin{cases}
      \recurrenceDiam(\delta), & \text{if}\ 2 < \traversalDiam(\delta) \\
      \traversalDiam(\delta), & \text{otherwise}.
    \end{cases}
\]
\end{mydef}
Limiting the computation of $\recurrenceDiam$ as in the base case function $\basecasefun_1$ significantly reduces the bound computation time.
This leads to to more problems being successfully bounded as shown in column 4 of Table~\ref{table:boundsDoms}, compared to when $\recurrenceDiam$ is used.
The substantially improved running time is shown in Figure~\ref{fig:boundcompareTDvsRD}.
We also observe that bounds computed using $\recurrenceDiam$ as a base case function are exactly the same as those computed using $\basecasefun_1$, for all problems on which they terminate.
Thus, this improvement in running time does not come at the cost of looser bounds.
Indeed, we conjecture the following.
\begin{myconj}
\label{conj:tdeqrd112}
For any factored transition system $\delta$, if $\traversalDiam(\delta)\in\{0,1,2\}$, then $\traversalDiam(\delta)=\recurrenceDiam(\delta)$.
\end{myconj}

Note, however, that the number of problems successfully bounded when using $\basecasefun_1$ as a base case function is still less than the number of problems bounded using $\traversalDiam$.
This is because the large computation cost of $\recurrenceDiam$ on the base cases on which it is invoked is still much more than the cost of computing $\traversalDiam$.
Another technique to improve the bound computation time is to limit the computation of $\recurrenceDiam$ to problems whose state spaces' sizes are bound by a constant.
This is done with the following base case function.
\begin{mydef}
\[
\basecasefun_2(\delta) = 
\begin{cases}
  \basecasefun_1(\delta), & \text{if}\ 50<\expbound(\delta) \\
  \traversalDiam(\delta), & \text{otherwise}.
\end{cases}
\]
\end{mydef}
We set 50 as an upper limit on the state space size as more than 95\% of abstractions whose $\recurrenceDiam$ was successfully computed had values less than 50. 

As shown in column 5 of Table~\ref{table:boundsDoms}, the number of problems that are successfully bounded within 20 minutes when $\basecasefun_2$ is used as a base case function is substantially more than those when $\basecasefun_1$ is used, especially in the domains where $\recurrenceDiam$ and $\basecasefun_1$ were less successful than $\traversalDiam$.
However, the bounds computed when $\basecasefun_2$ is used are sometimes worse than those computed when $\basecasefun_1$ is used, like in the case of TPP.
Figure~\ref{fig:boundcompareTDvsRD} shows this bound degradation and bound computation time improvement for the problems on which both methods terminate.
Nonetheless, the bounds computed using $\basecasefun_2$ are still much better than $\traversalDiam$, as shown in Figure~\ref{fig:boundcompareTDvsRD}.
This is because there are abstractions whose recurrence diameter is computable within the timeout and whose state spaces have more than 50 states.
For those abstractions, $\traversalDiam$ is computed instead of $\recurrenceDiam$, when $\basecasefun_2$ is used.
An interesting problem is adjusting the threshold in $\basecasefun_2$ to maximise the number of abstractions whose recurrence diameter can be computed within the timeout.
We do not fully explore this problem here.

\paragraph{Using the bounds for SAT-based planning}

Table~\ref{table:boundsDoms} shows that the coverage of {\sc Mp} increases if we use, as horizons, the bounds computed with base case functions involving $\recurrenceDiam$, compared to when using $\traversalDiam$ as a base case function.
An exception is Zeno, where the better bound computation time using $\traversalDiam$ is decisive.
In addition to coverage, it makes sense to take a look into how the exact running times compare as many problems are solved using any of the bounds.
The plots at the bottom two rows right of Figure~\ref{fig:boundcompareTDvsRD} show the time needed for computing a plan when using different pairs of base case functions for problems on which both methods succeed.
The plots show that the planning time (excluding bound computation) using the tighter bounds is much smaller almost always, which is to be expected.
The other shows planning time including bound computation.
Interestingly, although the bound computation time is more, e.g.\ when using $\basecasefun_1$ than $\traversalDiam$, tighter bounds always payoff for problems that need more than 5 seconds to solve.
Thus, time spent computing better bounds before planning is almost always well-spent.

 \section{Conclusion}

The recurrence diameter was identified by many authors as an upper bound on transition sequence lengths in the areas of verification~\cite{baumgartner2002property,KroeningS03,kroening2011linear} and AI~planning~\cite{abdulaziz2015verified,icaps2017,abdulaziz:2019}.
However, previous authors noted that computing it is not practically useful, as it can take exponentially longer than solving the underlying planning or model checking problem.
Nonetheless, we show that, indeed, computing the recurrence diameter is useful when used for compositional bounding.
We do so using a SMT encoding that exploits the factored state space representation, and by combining the recurrence diameter with the traversal diameter, which is an easier to compute topological property.

Broad directions for future work include \begin{enumerate*} \item devising methods to calculate base case functions that are tighter than the recurrence diameter and \item devising compositional methods that can compute allow better problem decompositions than those by \citeauthor{abdulaziz2015verified}~\citeyear{abdulaziz2015verified,icaps2017}.\end{enumerate*}
A third more interesting direction is devising methods that use \emph{radius} concepts for compositional bounding, which are the topological properties corresponding to the different diameter concepts (i.e.\ the diameter, the traversal diameter, and the recurrence diameter) that consider paths starting only at the initial state.
Using these radius concepts for bounding has the potential to boost the efficiency of bound computation as well as bound tightness.
A significant challenge is that the theory justifying existing compositional bounding methods fully relies on the fact that the base case function is a diameter concept (in particular, the \emph{sublist diameter})~\cite{abdulaziz2015verified}.

\section{Acknowledgements}
We would like to thank Dr.\ Charles Gretton for the very helpful discussions and comments on this paper.
We also thank the anonymous reviewers whose comments helped improve this paper.
We also thank the German Research Foundation for funding that facilitated this work through the DFG Koselleck Grant NI 491/16-1.
\bibliography{short_paper}

\begin{thebibliography}{29}
\providecommand{\natexlab}[1]{#1}
\providecommand{\url}[1]{\texttt{#1}}
\providecommand{\urlprefix}{URL }
\expandafter\ifx\csname urlstyle\endcsname\relax
  \providecommand{\doi}[1]{doi:\discretionary{}{}{}#1}\else
  \providecommand{\doi}{doi:\discretionary{}{}{}\begingroup
  \urlstyle{rm}\Url}\fi

\bibitem[{Abboud, Williams, and Wang(2016)}]{abboud2016approximation}
Abboud, A.; Williams, V.~V.; and Wang, J. 2016.
\newblock {Approximation and Fixed Parameter Subquadratic Algorithms for Radius
  and Diameter in Sparse Graphs}.
\newblock In \emph{SODA}.

\bibitem[{Abdulaziz(2017)}]{abdulaziz2017formally}
Abdulaziz, M. 2017.
\newblock \emph{{Formally Verified Compositional Algorithms for Factored
  Transition Systems}}.
\newblock The Australian National University.

\bibitem[{Abdulaziz(2019)}]{abdulaziz:2019}
Abdulaziz, M. 2019.
\newblock Plan-Length Bounds: Beyond 1-way Dependency.
\newblock In \emph{AAAI}.

\bibitem[{Abdulaziz, Gretton, and Norrish(2015)}]{abdulaziz2015verified}
Abdulaziz, M.; Gretton, C.; and Norrish, M. 2015.
\newblock {Verified Over-Approximation of the Diameter of Propositionally
  Factored Transition Systems}.
\newblock In \emph{{ITP}}.

\bibitem[{Abdulaziz, Gretton, and Norrish(2017)}]{icaps2017}
Abdulaziz, M.; Gretton, C.; and Norrish, M. 2017.
\newblock {A State Space Acyclicity Property for Exponentially Tighter Plan
  Length Bounds}.
\newblock In \emph{ICAPS}.

\bibitem[{Aingworth et~al.(1999)Aingworth, Chekuri, Indyk, and
  Motwani}]{aingworth1999fast}
Aingworth, D.; Chekuri, C.; Indyk, P.; and Motwani, R. 1999.
\newblock {Fast Estimation of Diameter and Shortest Paths (Without Matrix
  Multiplication)}.
\newblock \emph{SICOMP} .

\bibitem[{Alon, Galil, and Margalit(1997)}]{alon1997exponent}
Alon, N.; Galil, Z.; and Margalit, O. 1997.
\newblock On the exponent of the all pairs shortest path problem.
\newblock \emph{JCSS} .

\bibitem[{Baumgartner, Kuehlmann, and Abraham(2002)}]{baumgartner2002property}
Baumgartner, J.; Kuehlmann, A.; and Abraham, J. 2002.
\newblock {Property Checking Via Structural Analysis}.
\newblock In \emph{CAV}.

\bibitem[{Biere et~al.(1999)Biere, Cimatti, Clarke, and Zhu}]{BiereCCZ99}
Biere, A.; Cimatti, A.; Clarke, E.~M.; and Zhu, Y. 1999.
\newblock {Symbolic Model Checking without {BDDs}}.
\newblock In \emph{{TACAS}}.

\bibitem[{Chan(2010)}]{chan2010more}
Chan, T.~M. 2010.
\newblock {More Algorithms for All-Pairs Shortest Paths in Weighted Graphs}.
\newblock \emph{SICOMP} .

\bibitem[{Chechik et~al.(2014)Chechik, Larkin, Roditty, Schoenebeck, Tarjan,
  and Williams}]{chechik2014better}
Chechik, S.; Larkin, D.~H.; Roditty, L.; Schoenebeck, G.; Tarjan, R.~E.; and
  Williams, V.~V. 2014.
\newblock {Better Approximation Algorithms for the Graph Diameter}.
\newblock In \emph{{SODA}}.

\bibitem[{Dutertre(2014)}]{DBLP:conf/cav/Dutertre14}
Dutertre, B. 2014.
\newblock Yices 2.2.
\newblock In \emph{CAV}.

\bibitem[{Fikes and Nilsson(1971)}]{fikes1971strips}
Fikes, R.~E.; and Nilsson, N.~J. 1971.
\newblock {STRIPS: A New Approach to the Application of Theorem Proving to
  Problem Solving}.
\newblock \emph{AI} .

\bibitem[{Fredman(1976)}]{fredman1976new}
Fredman, M.~L. 1976.
\newblock {New Bounds on the Complexity of the Shortest Path Problem}.
\newblock \emph{SICOMP} .

\bibitem[{Gerevini, Saetti, and
  Vallati(2015)}]{DBLP:journals/aicom/GereviniSV15}
Gerevini, A.~E.; Saetti, A.; and Vallati, M. 2015.
\newblock {Exploiting Macro-Actions and Predicting Plan Length in Planning as
  Satisfiability}.
\newblock \emph{{AI} Communications} .

\bibitem[{Hemaspaandra et~al.(2010)Hemaspaandra, Hemaspaandra, Tantau, and
  Watanabe}]{hemaspaandra2010complexity}
Hemaspaandra, E.; Hemaspaandra, L.~A.; Tantau, T.; and Watanabe, O. 2010.
\newblock {On the Complexity of Kings}.
\newblock \emph{TCS} .

\bibitem[{Kautz and Selman(1992)}]{kautz:selman:92}
Kautz, H.~A.; and Selman, B. 1992.
\newblock Planning as Satisfiability.
\newblock In \emph{ECAI}.

\bibitem[{Knoblock(1994)}]{knoblock:94}
Knoblock, C.~A. 1994.
\newblock {Automatically Generating Abstractions for Planning}.
\newblock \emph{AI} .

\bibitem[{Knuth(1998)}]{DBLP:books/lib/Knuth98a}
Knuth, D.~E. 1998.
\newblock \emph{{The Art of Computer Programming, Volume III, 2nd Edition}}.
\newblock Addison-Wesley.

\bibitem[{Kroening et~al.(2011)Kroening, Ouaknine, Strichman, Wahl, and
  Worrell}]{kroening2011linear}
Kroening, D.; Ouaknine, J.; Strichman, O.; Wahl, T.; and Worrell, J. 2011.
\newblock {Linear Completeness Thresholds for Bounded Model Checking}.
\newblock In \emph{CAV}.

\bibitem[{Kroening and Strichman(2003)}]{KroeningS03}
Kroening, D.; and Strichman, O. 2003.
\newblock Efficient Computation of Recurrence Diameters.
\newblock In \emph{{VMCAI}}.

\bibitem[{McMillan(1993)}]{mcmillan1993symbolic}
McMillan, K.~L. 1993.
\newblock {Symbolic Model Checking}.
\newblock Kluwer.

\bibitem[{Papadimitriou and Yannakakis(1986)}]{papadimitriou1986note}
Papadimitriou, C.~H.; and Yannakakis, M. 1986.
\newblock {A Note on Succinct Representations of Graphs}.
\newblock \emph{Information and Control} .

\bibitem[{Pardalos and Migdalas(2004)}]{pardalos2004note}
Pardalos, P.~M.; and Migdalas, A. 2004.
\newblock {A Note on the Complexity of Longest Path Problems Related to Graph
  Coloring}.
\newblock \emph{Applied Mathematics Letters} .

\bibitem[{Rintanen(2012)}]{rintanen:12}
Rintanen, J. 2012.
\newblock {Planning as Satisfiability: Heuristics}.
\newblock \emph{AI} .

\bibitem[{Rintanen and Gretton(2013)}]{Rintanen:Gretton:2013}
Rintanen, J.; and Gretton, C.~O. 2013.
\newblock {Computing Upper Bounds on Lengths of Transition Sequences}.
\newblock In \emph{IJCAI}.

\bibitem[{Roditty and Vassilevska~Williams(2013)}]{roditty2013fast}
Roditty, L.; and Vassilevska~Williams, V. 2013.
\newblock {Fast Approximation Algorithms for the Diameter and Radius of Sparse
  Graphs}.
\newblock In \emph{{STOC}}.

\bibitem[{Williams and Nayak(1997)}]{williams:nayak:97}
Williams, B.~C.; and Nayak, P.~P. 1997.
\newblock {A Reactive Planner for a Model-based Executive}.
\newblock In \emph{IJCAI}.

\bibitem[{Yuster(2010)}]{yuster2010computing}
Yuster, R. 2010.
\newblock {Computing the diameter polynomially faster than APSP}.
\newblock \emph{arXiv preprint arXiv:1011.6181} .

\end{thebibliography}
\end{document}